\documentclass{article}

\def\confversion{0}

\usepackage{ifthen}

\newcommand{\ignore}[1]{}

\ifthenelse{\equal{\confversion}{1}}
{
	\newcommand{\conf}[1]{#1}
}
{
	\newcommand{\conf}[1]{\ignore{#1}}	
}
\ifthenelse{\equal{\confversion}{0}}
{
	\newcommand{\full}[1]{#1}
}
{
	\newcommand{\full}[1]{\ignore{#1}}	
}

\conf{\usepackage{neurips_2022}}
\full{\usepackage[preprint,nonatbib]{neurips_2022}}

\title{Improved Approximation for Fair Correlation Clustering}

\author{%
Sara Ahmadian\\
Google Research\\
\texttt{sahmadian@google.com}\\
\And
Maryam Negahbani\\
Dartmouth College\\
\texttt{maryam@cs.dartmouth.edu}\\

}

\usepackage{xspace}

\usepackage[dvipsnames]{xcolor}

\usepackage{array,multirow,multicol}
\usepackage{mathdots}
\usepackage{amsthm,amssymb,amsmath}
\usepackage{mathtools}
\usepackage{mathrsfs}
\usepackage{enumerate}
\usepackage{siunitx}
\usepackage{paralist}
\usepackage{makecell}
\usepackage{url}            
\usepackage{booktabs}       
\usepackage{microtype}      
\usepackage{graphicx}
\usepackage[pdfencoding=auto]{hyperref}
\usepackage[nameinlink,capitalize,noabbrev]{cleveref}

\usepackage{algorithm}
\usepackage{algpseudocode}

\newtheorem{theorem}{Theorem}
\theoremstyle{definition}
\newtheorem{definition}{Definition}
\newtheorem{fact}{Fact}

\newtheorem{lemma}{Lemma}
\newtheorem{corollary}{Corollary}

\newtheorem*{theorem*}{\bf Informal Theorem}

\crefname{Fact}{Fact}{Facts}

\usepackage{thm-restate}

\usepackage{ucs} 
\usepackage{tikz}
\usepackage{tkz-tab}
\usepackage{caption}
\usepackage{latexsym}
\usepackage{subcaption}

\newcommand{\fcclp}{\textsf{FCC-LP}\xspace}
\newcommand{\ep}{{E^+}}
\newcommand{\en}{{E^-}}
\newcommand{\correlation}{\textsf{Correlation Clustering}\xspace}
\newcommand{\fcorr}{\textsf{Fair \correlation}\xspace}
\newcommand{\corrcost}{\textsf{corr}}

\newcommand{\etal}{{\em et al}.~}

\newcommand{\cC}{\mathcal{C}}

\newcommand{\lp}{\textsf{LP}\xspace}
\newcommand{\ip}{\textsf{IP}\xspace}

\def\bank{{\tt bank}~}
\def\census{{\tt census}~}
\def\diabetes{{\tt diabetes}~}
\def\reuters{{\tt reuters}~}
\def\victorian{{\tt victorian}~}
\def\amazon{{\tt amazon}~}

\begin{document}

\maketitle

\begin{abstract}
Correlation clustering is a ubiquitous paradigm in unsupervised machine learning where addressing unfairness is a major challenge. Motivated by this, we study {\em fair correlation clustering} where the data points may belong to different protected groups and the goal is to ensure fair representation of all groups across clusters. Our paper significantly generalizes and improves on the quality guarantees of previous work of Ahmadi~\etal~\cite{ahmadi2020fair} and Ahmadian~\etal~\cite{ ahmadian2020fair} as follows. 
\begin{itemize}
    \item We allow the user to specify an arbitrary upper bound on the representation of each group in a cluster.
    \item Our algorithm allows individuals to have multiple protected features and ensure fairness simultaneously across them all.
    \item We prove guarantees for clustering quality and fairness in this general setting. Furthermore, this improves on the results for the special cases studied in previous work.
\end{itemize}
Our experiments on real-world data demonstrate that our clustering quality compared to the optimal solution is much better than what our theoretical result suggests. 
\end{abstract}

\section{Introduction}

Machine learning algorithms are used in many sensitive applications such as awarding home loans~\cite{khandani2010consumer,malhotra2003evaluating} and predicting recidivism~\cite{propublica,dressel2018accuracy,chouldechova2017fair}. Therefore, it is crucial to ensure these algorithms are \emph{fair} and are not biased towards or against some specific groups in the population. Defining and practicing fairness in machine learning and optimization has been a major trend in recent years~\cite{kamishima2011fairness,kamishima2012fairness,joseph2016fairness,celis2017ranking,celis2018multiwinner,chierichetti19a,yang2017measuring}. Clustering is one such learning paradigm with a line of work in fairness, starting with Chierichetti~\etal~\cite{CKLV17} and continuing with~\cite{BCFN19,ahmadian2019clustering,ahmadi2020fair,ahmadian2020fair}, to name a few.

\emph{Correlation clustering} is a popular unsupervised learning problem that has gained a lot of attention from both theory~\cite{bansal2004correlation, acn, JXLW20} and applied communities~\cite{van2007correlation, zheng2011coreference, cohen2002learning, mccallum2003toward} (see survey by \cite{Wirth2010} and references therein). In this problem, the input is a graph on a set of vertices (or nodes) corresponding to data entries along with similar($+$)/dissimilar($-$) labels on all pairs of nodes (i.e. labeled edges in a complete graph). The goal is to partition the nodes into so-called \emph{clusters} in a way that respects the given similarities the best: minimizing the \emph{clustering cost} defined as the total number of $+$ edges crossing clusters, in addition to $-$ edges inside clusters.
 


Ahmadi~\etal~\cite{ahmadi2020fair} and Ahmadian~\etal~\cite{ ahmadian2020fair} studied a variant of fair correlation clustering where each node has a color, and each color encodes a value of a \emph{protected feature}, e.g., red encodes woman for gender. 
In the case of $\ell$ colors and color 1 being the rarest one, $p_i$ is defined as the ratio of nodes of color $i$ to color $1$. Then the goal is to ensure the color distribution in clusters is the same as the entire data, while minimizing the clustering cost. \cite{ahmadi2020fair, ahmadian2020fair} design \emph{approximation algorithms} for this problem where a $\beta$-approximate clustering is one with cost at most $\beta$ times the cost of the optimal fair clustering. In particular, \cite{ahmadi2020fair} presented an $O(\ell^2 \max_{i} p_i^2)$-approximation algorithm and \cite{ahmadian2020fair} presented an $O(\ell^2)$-approximation when $p_i$'s are $1$. If the fairness constraint was instead, to ensure no cluster has a dominant color, one that takes over at least half of the cluster, \cite{ahmadian2020fair} present a $256$-approximation algorithm.

There are two main short-comings in both the previous work on \fcorr. \textbf{(1)} They do not cover the case where each node has \emph{multiple} protected attributes (e.g., gender, race, and age group), whereas, it is known~\cite{BCFN19} that ensuring fairness with respect to only one attribute and oblivious to others, can produce clusters that are extremely unfair with respect to the rest of the protected features; similar to when a standard color-oblivious clustering algorithm can output extremely unfair clusters~\cite{CKLV17,ahmadian2020fair}. \textbf{(2)} \cite{ahmadi2020fair, ahmadian2020fair} do not cover many natural and important settings of fairness constraints e.g. when no color is allowed to be more than 80\% of any cluster (see disparate impact doctrine~\cite{EEOC}, a rule in the United States).

We address these issues by designing an algorithm that allows the user to specify arbitrary upper bounds on the representation of each group in a cluster, where each individual node can have multiple protected features. We output clusters that are essentially fair across all these features simultaneously and bound the clustering cost of our algorithm with respect to the optimal solution, considerably improving the approximation ratios for the special cases studied in previous work.

\subsection{Our Result}
Our main contribution can be summarized as follows: We design an LP-rounding algorithm that given parameters $0 < \alpha_i \leq 1$ for colors $i \in \{1,\cdots,\ell\}$, and any small constant $\epsilon > 0$, returns a clustering with cost at most $O(\frac{1}{\epsilon \min_i{\alpha_i}})$ times that of the optimal fair solution. The clusters consist of singletons (that violate fairness by additive $+1$) and non-singleton clusters $C$, for which the number of points of color $i$ in $C$ is at most $(1+\epsilon) \alpha_i |C|$ for any $i$. To be more precise, we allow a $\max\{1,\epsilon|C|\max_i\alpha_i\}$ additive violation of the fairness constraints. See \Cref{corr:thm:main} for details and exact bound. 

Additive violation in group fairness is a recurring theme in {\em metric clustering}~\cite{BCFN19,ahmadian2019clustering,BerceaGKKRS019,}. 
We suspect for achieving our low approximation ratio, fairness violation is necessary due to to NP-hardness of special cases, similar to what is proved in~\cite{BCFN19}. Comparing with \cite{ahmadi2020fair, ahmadian2020fair} for the special cases addressed therein, in the case of $\alpha_i = 1/\ell$ our approximation ratio is better by a factor of $\ell$ (\Cref{corr:cor:propalpha}). For $\alpha$'s all equal to $1/2$ we get a $(4 + \frac{1}{\epsilon})$-approximation, a smaller constant compared to the 256-approximation of \cite{ahmadian2020fair} for $\varepsilon > 1/252$ (\Cref{corr:cor:alphahalf}).

Our empirical results in \Cref{corr:sec:experiments} demonstrate that our clustering cost is considerably better than the proven approximation ratio, namely, at most $15\%$ more than the optimal cost even for $\epsilon = 0.01$. Furthermore, our approximation ratio captures the tension between getting a low-cost clustering and having strict bounds on representation of points, due to small $\epsilon$ and $\alpha$'s. This relation of clustering cost with $\epsilon$ and $\alpha$'s is also apparent from our experiments in \Cref{corr:sec:experiments}.

Our algorithm is based on rounding a linear program (\lp) formulation of \fcorr, similar to the work of Ji \etal~\cite{JXLW20} on approximating other variants of correlation clustering. Our technical contribution is for cases where carving out low-cost clusters is at odds with ensuring fairness. This happens when fairness constraints in \lp are not effective due to integrality issues or when there are no clear cut fair and low-cost clusters in that region of the graph. See \Cref{corr:sec:degen} for details.

\subsection{Related Work}
\textbf{Fairness in machine learning and clustering.} 
Fairness in machine learning has received a lot of attention and is a fast growing literature (survey in \cite{caton2020fairness} and references therein). The efforts can be categorized into two main groups: (i) defining notions of fairness , and (ii) devising fair algorithms. Our work falls into the latter category and we concentrate on the notion of {\em disparate impact} which informally asks that the
decisions made (by an algorithm) should not be disproportionately different for applicants in different protected classes. Under this notion, there are works spanning from fair classification \cite{feldman2015certifying, zafar2017fairness}, to fair ranking problems \cite{celis2017ranking}, and to fair matroid optimization \cite{chierichetti19a}. Chierichetti \etal~\cite{CKLV17} introduced fair clustering problem based on this notion.

Chierichetti \etal~\cite{CKLV17} mainly defined fair clustering for two colors and later R{\"{o}}sner and Schmidt~\cite{RosnerS18} extended this definition to multiple colors. Both work required the distribution of colors in clusters to match the distribution of colors in the data and this definition was relaxed in Ahmadian~\etal~\cite{ahmadian2019clustering} by allowing arbitrary distribution of colors in different clusters as long as presence of each color in each cluster was bounded. Bera~\etal~\cite{BCFN19} further generalized this notion by allowing lower and upper bounds per groups and also allowing overlapping groups. Other closely related problems to fair clustering are clustering with diversity constraints~\cite{liyizhang}, fair center selection~\cite{chen2019proportionally}, and clustering with proportionality constraints~\cite{kleindessner2019fair}.

\textbf{Correlation clustering.} Bansal~\etal~\cite{bansal2004correlation} introduced and gave the first constant factor approximation for complete graphs with the current best being $~2.06$ by \cite{chawla2006hardness}. Variants of the problem include complete signed graph \cite{bansal2004correlation,acn}, and weighted graphs (generalizing incomplete signed graph) \cite{charikar2005clustering, demaine2006correlation}. The problem is shown to be APX-hard in the former case \cite{demaine2006correlation} and Unique-Games hard in the latter case \cite{chawla2006hardness}. The integrality gap\footnote{The maximum ratio between the solution quality of the integer program and its relaxation is called integrality gap of an \lp formulation.} of the \lp formulation of the problem for complete graphs is shown to be 2 \cite{charikar2005clustering}. Correlation clustering is also studied in the presence of various constraints. The most relevant is the upper-bounded correlation clustering, where each cluster is required to have size at most $M$ for a given input parameter $M$. \cite{JXLW20} presents a bicriteria algorithm for this version. 



\section[Definition and Preliminaries]{Problem Definition and Preliminaries}
The input of a \correlation problem is a complete undirected graph G = (V,E) with each edge $uv$ labeled either $+$ or $-$ based on whether $u$ and $v$ are similar or dissimilar, respectively. Let $\ep$ denote the set of positive edges and $\en$ denote the set of negative edges, so $E = \ep \cup \en$. For subsets $S, T \subseteq V$, $E(S,T)$ denotes the set of edges between vertices in $S$ and $T$, i.e., let $E(S,T) = E \cap (S\times T)$. We simplify this notation by defining $E(S) := E(S,S)$ and using $u$ instead of $\{u\}$ when applicable, e.g., $E(u,v) = E(\{u\}, \{v\})$. For a subset of edges $F \subseteq E$, let $F^+$ and $F^-$ denote the intersection of edges in $F$ with $\ep$ and $\en$ respectively, hence, $\ep(S,T) = \ep \cap E(S,T)$, and $\en(S,T) = \en \cap E(S,T)$. 

In a clustering problem, the goal is to find a partition $\cC$ of points such that points inside each cluster are similar and points in different clusters are dissimilar. In the \correlation problem, this notation is naturally extended so the goal is to minimize the total number of disagreements where a disagreement happens when two similar vertices are separated, i.e, a positive inter-cluster edge, or when two dissimilar vertices are clustered together, i.e., a negative intra-cluster edge. 
\begin{definition}\label{cor:def:correlation}
Given complete graph $G(V, \ep \cup \en)$, find a partition $\cC$ of $V$ that minimizes the {\em correlation cost} defined as follows
\begin{equation*}
    \corrcost(\cC) = \left|\bigcup_{C \in \cC} \en(C)\right| + \left|\bigcup_{C,C' \in \cC: C\neq C'} \ep(C,C')\right|.
\end{equation*}
\end{definition}
\noindent
For a subset $F$ of edges and a given clustering, we define 
$$
\corrcost(F) = \left|\bigcup_{C \in \cC} \en(C) \cap F\right| + \left|\bigcup_{C,C' \in \cC:\\ C\neq C'} \ep(C,C') \cap F\right|.
$$

The \fcorr is a generalization of \correlation problem where points may belong to groups corresponding to multiple protected features and there are constraints on representation of each group in each cluster. In this work, we consider the most general case with multiple features and different upper bound thresholds (suggested by \cite{BCFN19}) defined formally as follows.
\begin{definition}
In addition to the \correlation input, we are given a set of $\ell$ colors $V_1, V_2, \dots, V_\ell \subseteq V$ that may overlap. Given {\em fairness parameters} $\alpha_1, \alpha_2, \dots, \alpha_\ell \in [0,1]$, the goal is to find a clustering $\cC$ minimizing the \correlation cost while satisfying the {\em fairness constraint} that for any $C \in \cC$ and color $i \in \{0,\cdots,\ell\}$, $|V_i \cap C| \leq \alpha_i |C|$.
\end{definition}

\section[The Algorithm]{The Fair Correlation Clustering Algorithm}
In this section, we present our algorithm for solving \fcorr. The main idea is to first solve a linear program (\lp) relaxation of the problem to obtain a fractional solution and then use this fractional solution to form as many ``almost fair clusters'' as possible without sacrificing approximation factor by too much. Our aim is to get a {\em bicriterion} approximation factor, i.e., 
solution which will violate the fairness constraint mildly (say, $(1+\epsilon)$ factor) with cost at most $\beta$ times the optimal fractional solution which itself is at most the optimal cost. More precisely, for any input $\epsilon > 0$, we can show that each cluster $C$ of our algorithm is either of size $1$ or is {\em $\epsilon$-fair}, meaning $|V_i \cap C| \leq (1+\epsilon) \alpha_i |C|$ for any color class $V_i$, with approximation factor $\beta = O(\frac{1}{\epsilon \min_i{\alpha_i}})$.

\subsection{An \lp Formulation}
The standard \lp for \correlation has been studied extensively and heuristic approaches have been developed for solving it \cite{downing2010improved}. Our linear programming relaxation is just the extension of this \lp with fairness constraints and is the \lp relaxation (denoted by \fcclp) of the following integer program (\ip)
\begin{align}
   \min &\displaystyle\sum_{uv \in \ep} x_{uv} + \sum_{uv \in \en}(1-x_{uv}) & \label{corr:fcc-ip}\tag{FCC-\ip}\\
     &\displaystyle\sum_{v \in V_i} (1-x_{uv}) \leq \alpha_i\displaystyle\sum_{v \in V} (1-x_{uv})& \forall i \in [\ell], \forall u \in V, \tag{\lp-fair}\label{corr:lp:fair}\\
    & x_{uv} + x_{vw} \geq x_{uw} & \forall u,v,w \in V, \tag{$\bigtriangleup$-ineq}\label{corr:lp:trieq}\\
    & x_{uv} = x_{vu} & \forall u,v \in V, \nonumber\\
         & x_{uu} = 0, & \forall u \in V, \nonumber\\
    & x_{uv} \in \{0,1\} & \forall u,v\in V, \nonumber
\end{align}

Here, the indicator variable $x_{uv}$ denotes whether vertices $u$ and $v$ are assigned to different clusters or not; 0 and 1 values indicate same and different cluster respectively. Constraint (\ref{corr:lp:fair}) captures the fairness requirement as $\sum_{u\in V} (1 - x_{uv})$ is the size of cluster containing vertex $u$ and $\sum_{u\in V_i} (1 - x_{uv})$ is the number of vertices of color $i$ in this cluster. The rest of the constraints ensure that $x$ defines a distance metric (encoding three axioms for defining a metric). In particular, constraint (\ref{corr:lp:trieq}) also known as {\em triangle inequality} captures that if vertex $v$ and $w$ are assigned to different clusters then $u$ cannot be in the same cluster as both $v$ and $w$ at the same time. 

As mentioned, we solve the \lp relaxation of this \ip, called \fcclp, where for all $u,v \in V$ we allow $x_{uv}$ to take any (possibly fractional) value from $0$ to $1$. Observe that for any clustering, we can define a feasible $x$ that satisfies the constraints and the objective will be equal to the clustering cost, i.e., number of disagreements. Hence the optimal \lp cost lower bounds the optimal \ip cost. In fact, we get the claimed approximation by bounding the cost of our solution in terms of the optimal \lp cost.

For a subset of the edges $F \subseteq E$, we use the notation $\lp(F)$ to denote the \fcclp \emph{cost-share} of $F$. That is
\begin{equation*}
    \lp(F) := \sum_{uv \in F^+} x_{uv} + \sum_{uv\in F^-} (1-x_{uv}),
\end{equation*}
which we simplify to $\lp(uv)$ when $F = \{uv\}$.

\subsection{The Algorithm}
Our algorithm is based on rounding an optimal solution $x$ of the \fcorr \lp (\fcclp) which defines a metric on the vertices. It takes as input three parameters: $\epsilon > 0$, the allowed degree of violation in the fairness constraint, and parameters $0 < \rho \leq 1/2$ and $0 < \sigma \leq \rho/2$ which will be fixed later. The high-level idea of the algorithm is to carve out $\epsilon$-fair clusters as much as possible and then return all the remaining uncovered (unclustered) vertices as singleton clusters. We use the term {\em degenerate} to refer to such singleton clusters which will be collected in the set $\cC^1$ and we use the term {\em non-degenerate} to refer to non-singleton clusters, i.e., all $\epsilon$-fair clusters. 

For carving out a non-degenerate cluster, the algorithm relies on the distance metric defined by $x$ and only takes points that are in close proximity of each other. We basically look for a central point for which all the unclustered points at maximum distance $\rho$ from it, form an $\epsilon$-fair cluster and have average distance of at most $\sigma$. 
We refer to the latter condition as {\em density} condition and use the term {\em sparse} when this condition is not satisfied for a (prospective) cluster. This concludes the high-level idea of our algorithm presented in \Cref{corr:alg:main}.

\begin{algorithm}[tb]
   \caption{Fair-CC algorithm}
	\label{corr:alg:main}
\begin{algorithmic}[1]
   \State {\bfseries Input:} $G = (V, \ep \cup \en)$,  parameters $\epsilon, \sigma, \rho \in \mathbf{R^+}$ s.t. $2\sigma \leq \rho \leq .5$, \fcclp solution $\{x_{uv}: u,v \in V\}$
   \State {\bfseries Output:} Clustering $\cC$ including singletons in $\cC^1$
   \State $\cC \leftarrow \emptyset$ 
   \State $U \leftarrow V$ 
   \While{ $U \neq \emptyset$} 
        \State $T_u \leftarrow \{v \in U: x_{uv} \leq \rho\}$, $\forall u \in U$\label{corr:ln_tu} 
        \If{ $\exists u \in U: (\sum_{v \in T_u} x_{uv})/|T_u| \leq \sigma$  \text{ and } $(|V_i \cap T_u|) / |T_u| \leq (1+\epsilon) \alpha_i, \forall i$} \label{corr:ln:if}
            \State $\cC \leftarrow \cC \cup T_u$ 
	        \State $U \leftarrow U\backslash T_u$
	    \Else
	        \State $\cC^1 \leftarrow U$ \label{corr:alg:line-signleton}
	        \State $\cC \leftarrow \cC \cup  \bigcup _{u \in U}\{u\}$ 
	        \State $U \leftarrow \emptyset$ 
	   \EndIf
   \EndWhile
\end{algorithmic}
\end{algorithm}

\subsection{The Main Result}
To prove \cref{corr:alg:main} produces an approximately optimal solution, we bound the cost of edges by the \lp cost. Our main result is  the following: 
\begin{theorem}\label{corr:thm:main}
There is an \lp rounding algorithm that given an instance of \correlation and an $\epsilon > 0$, returns clustering $\cC$ such that for each $C \in \cC$ either $|C| = 1$ or $|V_i \cap C| \leq (1+\epsilon) \alpha_i |C|$ for all $i \in [\ell]$. This algorithm produces a $\beta$-approximation for
\begin{equation*}
    \beta := \max \Big\{\frac{1}{\epsilon\alpha^*}, 4 + \frac{1}{\epsilon}\Big\},
\end{equation*}
where $\alpha^* := \max_{i \in [\ell]} (1-\alpha_i)/\alpha_i$.
\end{theorem}

Note that for special cases studied by \cite{ahmadi2020fair,ahmadian2020fair}, we get the following improvements on approximation ratio modulo the fairness violation. 

\begin{corollary} \label{corr:cor:propalpha}
For the special case of $\alpha_i = \frac{p_i}{\sum_i p_i}$ for color classes $V_1,\cdots, V_l$ with $p_1 = 1$ and arbitrary choice of $p_i$ for $i \geq 2$, our algorithm gets an $\epsilon$-fair solution within $O(\epsilon^{-1}\sum_i p_i)$ times optimum which is bounded by $O(\epsilon^{-1} \ell \max_i p_i)$ factor of the optimal cost.
\end{corollary}

\begin{corollary} \label{corr:cor:alphahalf}
For the special case of $\alpha_i = \frac{1}{2}$ for color classes $V_1,\cdots, V_l$, our algorithm gets an $\epsilon$-fair solution within $(4 + \frac{1}{\epsilon})$-factor of the optimal cost.
\end{corollary}

In the next section, we prove \Cref{corr:thm:main},a direct followup of \Cref{corr:thm:non-degenerate,corr:thm:degenerate}.
\section{Analysis}\label{corr:sec:analysis}
In this section, we present the required ingredients for proving \cref{corr:thm:non-degenerate} and \cref{corr:thm:degenerate} for bounding the cost of non-degenerate and degenerate clusters in $\cC$. Roughly speaking non-degenerate clusters are more well-behaved and using the density property, i.e., bounded average distance, we can charge the correlation cost to the optimal \lp cost comfortably. The degenerate clusters require more building arguments based on various insights such as using the density of cluster around a point at the time of removal, and charging to points of the color class that violates the fairness constraint. The main idea is to charge the cost of a disagreement edge, a negative intra-cluster edge or a positive inter-cluster edge in $\cC$, to the \lp cost of a set of edges. As long as we can ensure that no edge gets charged too many times, we can bound the total cost of $\cC$ by the maximum factor an edge is charged. 

\subsection{Non-degenerate Clusters}\label{corr:sec:non-degen}
In this section, we look at non-degenerate clusters in $\cC$ and prove the following theorem which can be found in \cite{JXLW20} with minor changes in notation. But for the sake of completeness and making a gentler introduction to our contributions, \full{we bring the statements and proofs here. The purpose of this subsection is to prove the following theorem.}\conf{we bring an outline of the proofs here and include full proofs in \Cref{corr:subsec:missing}.}

\begin{theorem}\label{corr:thm:non-degenerate}
For $\cC$ output of \Cref{corr:alg:main}, the correlation cost of the set of edges incident to non-degenerate clusters, i.e., $F = \bigcup_{C \in \cC} E(C) \cup E(C,V\backslash C)$ is at most $\max\{\frac{1}{\rho-\sigma}, \frac{1}{1-\rho-\sigma}\} \lp(F)$.
\end{theorem}

We fix a non-degenerate cluster $C$ and use $U$ to denote the set of uncovered vertices at the time $C$ is formed. Let $T_u$ be the set corresponding to $C$ which includes all vertices in $U$ with maximum distance $\rho$ from $u$. Refer to \cref{fig:fig} for an accompanying diagram. We start with positive inter-cluster edges and do case analysis based on whether their endpoint outside of $T_u$ is close to the central vertex $u$ or not (see \cref{fig:fig}). Note that for any $v \in U\backslash T_u$, we have $x_{uv} > \rho$, but this is not enough for bounding the length of crossing edge between vertices in $T_u$ to $v$. We first look at the case that the endpoint is ``far enough'' from $u$.

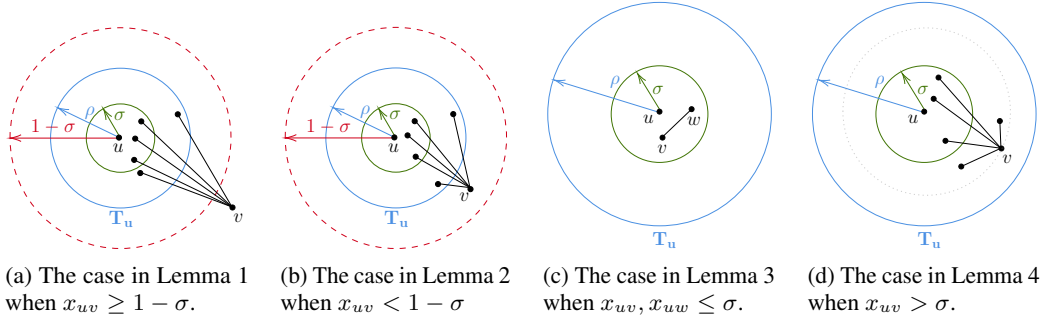
\begin{figure}
\centering
     \begin{subfigure}[b]{0.23\textwidth}
          \centering
          \resizebox{\linewidth}{!}{\tikzset{every picture/.style={line width=0.75pt}} 

\begin{tikzpicture}[x=0.75pt,y=0.75pt,yscale=-1,xscale=1]

\draw  [color={rgb, 255:red, 74; green, 144; blue, 226 }  ,draw opacity=1 ] (71,111) .. controls (71,72.89) and (101.89,42) .. (140,42) .. controls (178.11,42) and (209,72.89) .. (209,111) .. controls (209,149.11) and (178.11,180) .. (140,180) .. controls (101.89,180) and (71,149.11) .. (71,111) -- cycle ;
\draw  [color={rgb, 255:red, 65; green, 117; blue, 5 }  ,draw opacity=1 ] (106.25,111) .. controls (106.25,92.36) and (121.36,77.25) .. (140,77.25) .. controls (158.64,77.25) and (173.75,92.36) .. (173.75,111) .. controls (173.75,129.64) and (158.64,144.75) .. (140,144.75) .. controls (121.36,144.75) and (106.25,129.64) .. (106.25,111) -- cycle ;
\draw  [color={rgb, 255:red, 208; green, 2; blue, 27 }  ,draw opacity=1 ][dash pattern={on 4.5pt off 4.5pt}] (31,111) .. controls (31,50.8) and (79.8,2) .. (140,2) .. controls (200.2,2) and (249,50.8) .. (249,111) .. controls (249,171.2) and (200.2,220) .. (140,220) .. controls (79.8,220) and (31,171.2) .. (31,111) -- cycle ;
\draw [color={rgb, 255:red, 208; green, 2; blue, 27 }  ,draw opacity=1 ]   (140,111) -- (33,111) ;
\draw [shift={(31,111)}, rotate = 360] [color={rgb, 255:red, 208; green, 2; blue, 27 }  ,draw opacity=1 ][line width=0.75]    (10.93,-3.29) .. controls (6.95,-1.4) and (3.31,-0.3) .. (0,0) .. controls (3.31,0.3) and (6.95,1.4) .. (10.93,3.29)   ;
\draw [color={rgb, 255:red, 74; green, 144; blue, 226 }  ,draw opacity=1 ]   (140,111) -- (79.8,81.87) ;
\draw [shift={(78,81)}, rotate = 25.82] [color={rgb, 255:red, 74; green, 144; blue, 226 }  ,draw opacity=1 ][line width=0.75]    (10.93,-3.29) .. controls (6.95,-1.4) and (3.31,-0.3) .. (0,0) .. controls (3.31,0.3) and (6.95,1.4) .. (10.93,3.29)   ;
\draw [color={rgb, 255:red, 65; green, 117; blue, 5 }  ,draw opacity=1 ]   (140,111) -- (124.01,83.73) ;
\draw [shift={(123,82)}, rotate = 59.62] [color={rgb, 255:red, 65; green, 117; blue, 5 }  ,draw opacity=1 ][line width=0.75]    (10.93,-3.29) .. controls (6.95,-1.4) and (3.31,-0.3) .. (0,0) .. controls (3.31,0.3) and (6.95,1.4) .. (10.93,3.29)   ;
\draw  [fill={rgb, 255:red, 0; green, 0; blue, 0 }  ,fill opacity=1 ] (136.01,110.28) .. controls (136.13,108.91) and (137.34,107.89) .. (138.72,108.01) .. controls (140.09,108.13) and (141.11,109.34) .. (140.99,110.72) .. controls (140.87,112.09) and (139.66,113.11) .. (138.28,112.99) .. controls (136.91,112.87) and (135.89,111.66) .. (136.01,110.28) -- cycle ;
\draw  [fill={rgb, 255:red, 0; green, 0; blue, 0 }  ,fill opacity=1 ] (248.01,179.28) .. controls (248.13,177.91) and (249.34,176.89) .. (250.72,177.01) .. controls (252.09,177.13) and (253.11,178.34) .. (252.99,179.72) .. controls (252.87,181.09) and (251.66,182.11) .. (250.28,181.99) .. controls (248.91,181.87) and (247.89,180.66) .. (248.01,179.28) -- cycle ;
\draw  [fill={rgb, 255:red, 0; green, 0; blue, 0 }  ,fill opacity=1 ] (194.01,87.28) .. controls (194.13,85.91) and (195.34,84.89) .. (196.72,85.01) .. controls (198.09,85.13) and (199.11,86.34) .. (198.99,87.72) .. controls (198.87,89.09) and (197.66,90.11) .. (196.28,89.99) .. controls (194.91,89.87) and (193.89,88.66) .. (194.01,87.28) -- cycle ;

\draw  [fill={rgb, 255:red, 0; green, 0; blue, 0 }  ,fill opacity=1 ] (158.01,94.28) .. controls (158.13,92.91) and (159.34,91.89) .. (160.72,92.01) .. controls (162.09,92.13) and (163.11,93.34) .. (162.99,94.72) .. controls (162.87,96.09) and (161.66,97.11) .. (160.28,96.99) .. controls (158.91,96.87) and (157.89,95.66) .. (158.01,94.28) -- cycle ;
\draw  [fill={rgb, 255:red, 0; green, 0; blue, 0 }  ,fill opacity=1 ] (152.01,111.28) .. controls (152.13,109.91) and (153.34,108.89) .. (154.72,109.01) .. controls (156.09,109.13) and (157.11,110.34) .. (156.99,111.72) .. controls (156.87,113.09) and (155.66,114.11) .. (154.28,113.99) .. controls (152.91,113.87) and (151.89,112.66) .. (152.01,111.28) -- cycle ;
\draw  [fill={rgb, 255:red, 0; green, 0; blue, 0 }  ,fill opacity=1 ] (151.01,132.28) .. controls (151.13,130.91) and (152.34,129.89) .. (153.72,130.01) .. controls (155.09,130.13) and (156.11,131.34) .. (155.99,132.72) .. controls (155.87,134.09) and (154.66,135.11) .. (153.28,134.99) .. controls (151.91,134.87) and (150.89,133.66) .. (151.01,132.28) -- cycle ;
\draw  [fill={rgb, 255:red, 0; green, 0; blue, 0 }  ,fill opacity=1 ] (157.01,145.28) .. controls (157.13,143.91) and (158.34,142.89) .. (159.72,143.01) .. controls (161.09,143.13) and (162.11,144.34) .. (161.99,145.72) .. controls (161.87,147.09) and (160.66,148.11) .. (159.28,147.99) .. controls (157.91,147.87) and (156.89,146.66) .. (157.01,145.28) -- cycle ;
\draw    (196.5,87.5) -- (250,179) ;
\draw    (160.5,94.5) -- (250,179) ;
\draw    (154.5,111.5) --(250,179) ;
\draw    (153.5,132.5) -- (250,179) ;
\draw    (161.99,145.72) -- (250,179) ;

\draw (132.3,83) node [anchor=north west][inner sep=0.75pt]  [font=\large,color={rgb, 255:red, 65; green, 117; blue, 5 }] [font=\Large] {$\sigma $};
\draw (103.3,78) node [anchor=north west][inner sep=0.75pt]  [font=\large,color={rgb, 255:red, 74; green, 144; blue, 226 }]  [font=\Large] {$\rho $};
\draw (50,91) node [anchor=north west][inner sep=0.75pt]  [font=\large,color={rgb, 255:red, 208; green, 2; blue, 27 } ] [font=\Large] {$1-\sigma $};
\draw (127,181.4) node [anchor=north west][inner sep=0.75pt]  [font=\Large]  {$\textcolor[rgb]{0.29,0.56,0.89}{\mathbf{T_{u}}}$};
\draw (131.01,115) node [anchor=north west][inner sep=0.75pt]  [font=\Large]  {$u$};
\draw (250.01,184) node [anchor=north west][inner sep=0.75pt]  [font=\Large]  {$v$};
\end{tikzpicture}}  
          \caption{The case in \cref{corr:lem:non-long-pos} when $x_{uv} \geq 1-\sigma$. }
          \label{fig:A}
     \end{subfigure}
     \hspace{.27cm}
     \begin{subfigure}[b]{0.218\textwidth}
          \centering
          \resizebox{\linewidth}{!}{\tikzset{every picture/.style={line width=0.75pt}} 

\begin{tikzpicture}[x=0.75pt,y=0.75pt,yscale=-1,xscale=1]

\draw  [color={rgb, 255:red, 74; green, 144; blue, 226 }  ,draw opacity=1 ] (71,111) .. controls (71,72.89) and (101.89,42) .. (140,42) .. controls (178.11,42) and (209,72.89) .. (209,111) .. controls (209,149.11) and (178.11,180) .. (140,180) .. controls (101.89,180) and (71,149.11) .. (71,111) -- cycle ;
\draw  [color={rgb, 255:red, 65; green, 117; blue, 5 }  ,draw opacity=1 ] (106.25,111) .. controls (106.25,92.36) and (121.36,77.25) .. (140,77.25) .. controls (158.64,77.25) and (173.75,92.36) .. (173.75,111) .. controls (173.75,129.64) and (158.64,144.75) .. (140,144.75) .. controls (121.36,144.75) and (106.25,129.64) .. (106.25,111) -- cycle ;
\draw  [color={rgb, 255:red, 208; green, 2; blue, 27 }  ,draw opacity=1 ][dash pattern={on 4.5pt off 4.5pt}] (31,111) .. controls (31,50.8) and (79.8,2) .. (140,2) .. controls (200.2,2) and (249,50.8) .. (249,111) .. controls (249,171.2) and (200.2,220) .. (140,220) .. controls (79.8,220) and (31,171.2) .. (31,111) -- cycle ;
\draw [color={rgb, 255:red, 208; green, 2; blue, 27 }  ,draw opacity=1 ]   (140,111) -- (33,111) ;
\draw [shift={(31,111)}, rotate = 360] [color={rgb, 255:red, 208; green, 2; blue, 27 }  ,draw opacity=1 ][line width=0.75]    (10.93,-3.29) .. controls (6.95,-1.4) and (3.31,-0.3) .. (0,0) .. controls (3.31,0.3) and (6.95,1.4) .. (10.93,3.29)   ;
\draw [color={rgb, 255:red, 74; green, 144; blue, 226 }  ,draw opacity=1 ]   (140,111) -- (79.8,81.87) ;
\draw [shift={(78,81)}, rotate = 25.82] [color={rgb, 255:red, 74; green, 144; blue, 226 }  ,draw opacity=1 ][line width=0.75]    (10.93,-3.29) .. controls (6.95,-1.4) and (3.31,-0.3) .. (0,0) .. controls (3.31,0.3) and (6.95,1.4) .. (10.93,3.29)   ;
\draw [color={rgb, 255:red, 65; green, 117; blue, 5 }  ,draw opacity=1 ]   (140,111) -- (124.01,83.73) ;
\draw [shift={(123,82)}, rotate = 59.62] [color={rgb, 255:red, 65; green, 117; blue, 5 }  ,draw opacity=1 ][line width=0.75]    (10.93,-3.29) .. controls (6.95,-1.4) and (3.31,-0.3) .. (0,0) .. controls (3.31,0.3) and (6.95,1.4) .. (10.93,3.29)   ;
\draw  [fill={rgb, 255:red, 0; green, 0; blue, 0 }  ,fill opacity=1 ] (136.01,110.28) .. controls (136.13,108.91) and (137.34,107.89) .. (138.72,108.01) .. controls (140.09,108.13) and (141.11,109.34) .. (140.99,110.72) .. controls (140.87,112.09) and (139.66,113.11) .. (138.28,112.99) .. controls (136.91,112.87) and (135.89,111.66) .. (136.01,110.28) -- cycle ;
\draw  [fill={rgb, 255:red, 0; green, 0; blue, 0 }  ,fill opacity=1 ] (211.01,161.28) .. controls (211.13,159.91) and (212.34,158.89) .. (213.72,159.01) .. controls (215.09,159.13) and (216.11,160.34) .. (215.99,161.72) .. controls (215.87,163.09) and (214.66,164.11) .. (213.28,163.99) .. controls (211.91,163.87) and (210.89,162.66) .. (211.01,161.28) -- cycle ;
\draw  [fill={rgb, 255:red, 0; green, 0; blue, 0 }  ,fill opacity=1 ] (194.01,87.28) .. controls (194.13,85.91) and (195.34,84.89) .. (196.72,85.01) .. controls (198.09,85.13) and (199.11,86.34) .. (198.99,87.72) .. controls (198.87,89.09) and (197.66,90.11) .. (196.28,89.99) .. controls (194.91,89.87) and (193.89,88.66) .. (194.01,87.28) -- cycle ;
\draw    (196.5,87.5) -- (213,160) ;
\draw    (150,107) -- (213,160) ;
\draw    (181,156) -- (213,160) ;
\draw    (160,130) -- (213,160) ;
\draw    (159,94) -- (213,160) ;
\draw  [fill={rgb, 255:red, 0; green, 0; blue, 0 }  ,fill opacity=1 ] (149.01,108.24) .. controls (149.16,106.86) and (150.39,105.87) .. (151.76,106.01) .. controls (153.14,106.16) and (154.13,107.39) .. (153.99,108.76) .. controls (153.84,110.14) and (152.61,111.13) .. (151.24,110.99) .. controls (149.86,110.84) and (148.87,109.61) .. (149.01,108.24) -- cycle ;
\draw  [fill={rgb, 255:red, 0; green, 0; blue, 0 }  ,fill opacity=1 ] (156,94.61) .. controls (155.94,93.23) and (157.01,92.06) .. (158.39,92) .. controls (159.77,91.94) and (160.94,93.01) .. (161,94.39) .. controls (161.06,95.77) and (159.99,96.94) .. (158.61,97) .. controls (157.23,97.06) and (156.06,95.99) .. (156,94.61) -- cycle ;
\draw  [fill={rgb, 255:red, 0; green, 0; blue, 0 }  ,fill opacity=1 ] (157.01,129.24) .. controls (157.16,127.86) and (158.39,126.87) .. (159.76,127.01) .. controls (161.14,127.16) and (162.13,128.39) .. (161.99,129.76) .. controls (161.84,131.14) and (160.61,132.13) .. (159.24,131.99) .. controls (157.86,131.84) and (156.87,130.61) .. (157.01,129.24) -- cycle ;
\draw  [fill={rgb, 255:red, 0; green, 0; blue, 0 }  ,fill opacity=1 ] (179.01,156.24) .. controls (179.16,154.86) and (180.39,153.87) .. (181.76,154.01) .. controls (183.14,154.16) and (184.13,155.39) .. (183.99,156.76) .. controls (183.84,158.14) and (182.61,159.13) .. (181.24,158.99) .. controls (179.86,158.84) and (178.87,157.61) .. (179.01,156.24) -- cycle ;

\draw (132.3,83) node [anchor=north west][inner sep=0.75pt]  [font=\large,color={rgb, 255:red, 65; green, 117; blue, 5 }] [font=\Large] {$\sigma $};
\draw (103.3,78) node [anchor=north west][inner sep=0.75pt]  [font=\large,color={rgb, 255:red, 74; green, 144; blue, 226 }]  [font=\Large] {$\rho $};
\draw (50,91) node [anchor=north west][inner sep=0.75pt]  [font=\large,color={rgb, 255:red, 208; green, 2; blue, 27 } ] [font=\Large] {$1-\sigma $};
\draw (127,181.4) node [anchor=north west][inner sep=0.75pt]  [font=\Large]  {$\textcolor[rgb]{0.29,0.56,0.89}{\mathbf{T_{u}}}$};
\draw (131.01,115) node [anchor=north west][inner sep=0.75pt]  [font=\Large]  {$u$};
\draw (201.01,163) node [anchor=north west][inner sep=0.75pt]  [font=\Large]  {$v$};

\end{tikzpicture}}  
          \caption{The case in \cref{corr:lem:non-short-pos}  when $x_{uv} < 1-\sigma$}
          \label{fig:B}
     \end{subfigure}
     \hspace{.27cm}
  \begin{subfigure}[b]{0.22\textwidth}
          \centering
          \resizebox{\linewidth}{!}{\tikzset{every picture/.style={line width=0.75pt}} 

\begin{tikzpicture}[x=0.75pt,y=0.75pt,yscale=-1,xscale=1]

\draw  [color={rgb, 255:red, 74; green, 144; blue, 226 }  ,draw opacity=1 ] (277.5,130.5) .. controls (277.5,68.09) and (328.09,17.5) .. (390.5,17.5) .. controls (452.91,17.5) and (503.5,68.09) .. (503.5,130.5) .. controls (503.5,192.91) and (452.91,243.5) .. (390.5,243.5) .. controls (328.09,243.5) and (277.5,192.91) .. (277.5,130.5) -- cycle ;
\draw  [color={rgb, 255:red, 65; green, 117; blue, 5 }  ,draw opacity=1 ] (341.63,130.5) .. controls (341.63,103.51) and (363.51,81.63) .. (390.5,81.63) .. controls (417.49,81.63) and (439.38,103.51) .. (439.38,130.5) .. controls (439.38,157.49) and (417.49,179.38) .. (390.5,179.38) .. controls (363.51,179.38) and (341.63,157.49) .. (341.63,130.5) -- cycle ;
\draw [color={rgb, 255:red, 74; green, 144; blue, 226 }  ,draw opacity=1 ]   (387.95,128.06) -- (285.91,96.59) ;
\draw [shift={(284,96)}, rotate = 17.14] [color={rgb, 255:red, 74; green, 144; blue, 226 }  ,draw opacity=1 ][line width=0.75]    (10.93,-3.29) .. controls (6.95,-1.4) and (3.31,-0.3) .. (0,0) .. controls (3.31,0.3) and (6.95,1.4) .. (10.93,3.29)   ;
\draw [color={rgb, 255:red, 65; green, 117; blue, 5 }  ,draw opacity=1 ]   (390.39,125.5) -- (369.05,90.7) ;
\draw [shift={(368,89)}, rotate = 58.47] [color={rgb, 255:red, 65; green, 117; blue, 5 }  ,draw opacity=1 ][line width=0.75]    (10.93,-3.29) .. controls (6.95,-1.4) and (3.31,-0.3) .. (0,0) .. controls (3.31,0.3) and (6.95,1.4) .. (10.93,3.29)   ;
\draw  [fill={rgb, 255:red, 0; green, 0; blue, 0 }  ,fill opacity=1 ] (387.95,128.06) .. controls (387.92,126.68) and (389.01,125.53) .. (390.39,125.5) .. controls (391.77,125.47) and (392.91,126.57) .. (392.94,127.95) .. controls (392.97,129.33) and (391.88,130.47) .. (390.5,130.5) .. controls (389.12,130.53) and (387.98,129.44) .. (387.95,128.06) -- cycle ;
\draw    (392,155) -- (422.89,125.5) ;
\draw  [fill={rgb, 255:red, 0; green, 0; blue, 0 }  ,fill opacity=1 ] (420.39,125.61) .. controls (420.33,124.23) and (421.4,123.07) .. (422.78,123) .. controls (424.16,122.94) and (425.33,124.01) .. (425.39,125.39) .. controls (425.45,126.77) and (424.38,127.94) .. (423,128) .. controls (421.62,128.06) and (420.45,126.99) .. (420.39,125.61) -- cycle ;
\draw  [fill={rgb, 255:red, 0; green, 0; blue, 0 }  ,fill opacity=1 ] (391,154) .. controls (391.15,152.63) and (392.38,151.63) .. (393.75,151.78) .. controls (395.12,151.92) and (396.12,153.15) .. (395.97,154.53) .. controls (395.83,155.9) and (394.6,156.9) .. (393.22,156.75) .. controls (391.85,156.6) and (390.85,155.37) .. (391,154) -- cycle ;

\draw (382,100) node [anchor=north west][inner sep=0.75pt]  [font=\large,color={rgb, 255:red, 65; green, 117; blue, 5 }]  [font=\Large]  {$\sigma $};
\draw (339,90) node [anchor=north west][inner sep=0.75pt]  [font=\large,color={rgb, 255:red, 74; green, 144; blue, 226 }]  [font=\Large] {$\rho $};
\draw (382,248.4) node [anchor=north west][inner sep=0.75pt]  [font=\Large]  {$\mathbf{\textcolor[rgb]{0.29,0.56,0.89}{T}\textcolor[rgb]{0.29,0.56,0.89}{_{u}}}$};
\draw (372.5,130) node [anchor=north west][inner sep=0.75pt]  [font=\Large]  {$u$};
\draw (387.01,160) node [anchor=north west][inner sep=0.75pt]  [font=\Large]  {$v$};
\draw (417,132) node [anchor=north west][inner sep=0.75pt]  [font=\Large]  {$w$};

\end{tikzpicture}}  
          \caption{The case in \cref{corr:lem:non-short-neg} when \small {$x_{uv}, x_{uw} \leq \sigma$}. }
          \label{fig:C}
     \end{subfigure}
     \hspace{.27cm}
     \begin{subfigure}[b]{0.22\textwidth}
          \centering
          \resizebox{\linewidth}{!}{\tikzset{every picture/.style={line width=0.75pt}} 

\begin{tikzpicture}[x=0.75pt,y=0.75pt,yscale=-1,xscale=1]

\draw  [color={rgb, 255:red, 74; green, 144; blue, 226 }  ,draw opacity=1 ] (277.5,130.5) .. controls (277.5,68.09) and (328.09,17.5) .. (390.5,17.5) .. controls (452.91,17.5) and (503.5,68.09) .. (503.5,130.5) .. controls (503.5,192.91) and (452.91,243.5) .. (390.5,243.5) .. controls (328.09,243.5) and (277.5,192.91) .. (277.5,130.5) -- cycle ;
\draw  [color={rgb, 255:red, 65; green, 117; blue, 5 }  ,draw opacity=1 ] (341.63,130.5) .. controls (341.63,103.51) and (363.51,81.63) .. (390.5,81.63) .. controls (417.49,81.63) and (439.38,103.51) .. (439.38,130.5) .. controls (439.38,157.49) and (417.49,179.38) .. (390.5,179.38) .. controls (363.51,179.38) and (341.63,157.49) .. (341.63,130.5) -- cycle ;
\draw [color={rgb, 255:red, 74; green, 144; blue, 226 }  ,draw opacity=1 ]   (387.95,128.06) -- (285.91,96.59) ;
\draw [shift={(284,96)}, rotate = 17.14] [color={rgb, 255:red, 74; green, 144; blue, 226 }  ,draw opacity=1 ][line width=0.75]    (10.93,-3.29) .. controls (6.95,-1.4) and (3.31,-0.3) .. (0,0) .. controls (3.31,0.3) and (6.95,1.4) .. (10.93,3.29)   ;
\draw [color={rgb, 255:red, 65; green, 117; blue, 5 }  ,draw opacity=1 ]   (390.39,125.5) -- (369.05,90.7) ;
\draw [shift={(368,89)}, rotate = 58.47] [color={rgb, 255:red, 65; green, 117; blue, 5 }  ,draw opacity=1 ][line width=0.75]    (10.93,-3.29) .. controls (6.95,-1.4) and (3.31,-0.3) .. (0,0) .. controls (3.31,0.3) and (6.95,1.4) .. (10.93,3.29)   ;
\draw  [fill={rgb, 255:red, 0; green, 0; blue, 0 }  ,fill opacity=1 ] (387.95,128.06) .. controls (387.92,126.68) and (389.01,125.53) .. (390.39,125.5) .. controls (391.77,125.47) and (392.91,126.57) .. (392.94,127.95) .. controls (392.97,129.33) and (391.88,130.47) .. (390.5,130.5) .. controls (389.12,130.53) and (387.98,129.44) .. (387.95,128.06) -- cycle ;
\draw  [fill={rgb, 255:red, 0; green, 0; blue, 0 }  ,fill opacity=1 ] (402.39,93.61) .. controls (402.33,92.23) and (403.4,91.07) .. (404.78,91) .. controls (406.16,90.94) and (407.33,92.01) .. (407.39,93.39) .. controls (407.45,94.77) and (406.38,95.94) .. (405,96) .. controls (403.62,96.06) and (402.45,94.99) .. (402.39,93.61) -- cycle ;
\draw  [fill={rgb, 255:red, 0; green, 0; blue, 0 }  ,fill opacity=1 ] (466,165) .. controls (466.15,163.63) and (467.38,162.63) .. (468.75,162.78) .. controls (470.12,162.92) and (471.12,164.15) .. (470.97,165.53) .. controls (470.83,166.9) and (469.6,167.9) .. (468.22,167.75) .. controls (466.85,167.6) and (465.85,166.37) .. (466,165) -- cycle ;
\draw  [color={rgb, 255:red, 155; green, 155; blue, 155 }  ,draw opacity=0.52 ][dash pattern={on 0.84pt off 2.51pt}] (308.29,127.95) .. controls (308.29,81.19) and (346.19,43.29) .. (392.94,43.29) .. controls (439.7,43.29) and (477.6,81.19) .. (477.6,127.95) .. controls (477.6,174.7) and (439.7,212.6) .. (392.94,212.6) .. controls (346.19,212.6) and (308.29,174.7) .. (308.29,127.95) -- cycle ;
\draw  [fill={rgb, 255:red, 0; green, 0; blue, 0 }  ,fill opacity=1 ] (464.39,137.61) .. controls (464.33,136.23) and (465.4,135.07) .. (466.78,135) .. controls (468.16,134.94) and (469.33,136.01) .. (469.39,137.39) .. controls (469.45,138.77) and (468.38,139.94) .. (467,140) .. controls (465.62,140.06) and (464.45,138.99) .. (464.39,137.61) -- cycle ;
\draw  [fill={rgb, 255:red, 0; green, 0; blue, 0 }  ,fill opacity=1 ] (409.39,157.39) .. controls (409.46,156.01) and (410.63,154.94) .. (412.01,155.01) .. controls (413.39,155.07) and (414.45,156.24) .. (414.39,157.62) .. controls (414.32,159) and (413.15,160.06) .. (411.77,160) .. controls (410.39,159.94) and (409.33,158.77) .. (409.39,157.39) -- cycle ;
\draw  [fill={rgb, 255:red, 0; green, 0; blue, 0 }  ,fill opacity=1 ] (425.39,183.39) .. controls (425.46,182.01) and (426.63,180.94) .. (428.01,181.01) .. controls (429.39,181.07) and (430.45,182.24) .. (430.39,183.62) .. controls (430.32,185) and (429.15,186.06) .. (427.77,186) .. controls (426.39,185.94) and (425.33,184.77) .. (425.39,183.39) -- cycle ;
\draw    (427.89,183.5) -- (468,165) ;
\draw    (466.89,137.5) -- (468,165) ;
\draw    (404.89,93.5) -- (468,165) ;
\draw    (411.89,157.5) -- (468,165) ;
\draw    (400,115) -- (468,165) ;
\draw  [fill={rgb, 255:red, 0; green, 0; blue, 0 }  ,fill opacity=1 ] (397.5,115.11) .. controls (397.44,113.73) and (398.51,112.56) .. (399.89,112.5) .. controls (401.27,112.44) and (402.44,113.51) .. (402.5,114.89) .. controls (402.56,116.27) and (401.49,117.44) .. (400.11,117.5) .. controls (398.73,117.56) and (397.56,116.49) .. (397.5,115.11) -- cycle ;

\draw (382,100) node [anchor=north west][inner sep=0.75pt]  [font=\large,color={rgb, 255:red, 65; green, 117; blue, 5 }]  [font=\Large]  {$\sigma $};
\draw (339,90) node [anchor=north west][inner sep=0.75pt]  [font=\large,color={rgb, 255:red, 74; green, 144; blue, 226 }]  [font=\Large] {$\rho $};
\draw (382,248.4) node [anchor=north west][inner sep=0.75pt]  [font=\Large]  {$\mathbf{\textcolor[rgb]{0.29,0.56,0.89}{T}\textcolor[rgb]{0.29,0.56,0.89}{_{u}}}$};
\draw (372.5,130) node [anchor=north west][inner sep=0.75pt]  [font=\Large]  {$u$};
\draw (468,170) node [anchor=north west][inner sep=0.75pt]  [font=\Large]  {$v$};

\end{tikzpicture}}  
          \caption{The case in \cref{corr:lem:non-long-neg} when $x_{uv} > \sigma$.}
          \label{fig:D}
     \end{subfigure}
     \caption{Analysis of inter-cluster edges incident with vertices covered by non-degenerate clusters.
     First two figures, positive edges and last two figures, negative edges.}
     \label{fig:fig}
 \end{figure}

\begin{restatable}{lemma}{nonlongpos}\label{corr:lem:non-long-pos}
For any vertex $v \in U\backslash T_u$ with $x_{uv} \geq (1-\sigma)$, $|E^+(v, T_u)| \leq \frac{1}{1-\rho - \sigma}\lp(E^+(v, T_u))$.
\end{restatable}
\full{\begin{proof}
Fix an edge $vw \in E^+(v, T_u)$ ($w$ may or may not be equal to $u$). Roughly speaking, since $v$ is far from $u$ and $w$ is close to $u$ ($w$ belongs to $T_u$), the length of edge $vw$ cannot be too short. Formally, by \ref{corr:lp:trieq}, 
\begin{align}
x_{vw} \geq x_{uv} - x_{uw} \geq 1 - \sigma - \rho &\tag*{since  $x_{uw} \leq \rho$,} 
\end{align}
and so the cost of this edge can be just charged to the \lp cost the edge. 
\end{proof}}

The remaining set of positive inter-cluster edges correspond to vertices that are not ``far enough'' from $u$. Here the length of a crossing edge may be short and so we rely on the fact that the cluster $T_u$ is dense, i.e., $\sum_{w\in T_u} x_{uw} \leq \sigma|T_u|$, and so on average the length of crossing edges are long. In the following lemma, the number of these ``short'' positive edges from a $v$ to $T_u$ are bounded by the LP cost of \emph{all} edges from $v$ to $T_u$, including the negative ones.
\begin{restatable}{lemma}{nonshortpos}\label{corr:lem:non-short-pos}
For any vertex $v \in U\backslash T_u$ with $x_{uv}  < (1-\sigma)$, $|E^+(v, T_u)| \leq \frac{1}{\rho - \sigma}\lp(E(v, T_u))$.
\end{restatable}
\full{\begin{proof}
Let $P := \ep(v,T_u)$ with $p = |P|$ and let $N := \en(v,T_u)$ with $n = |N|$. Analyzing the \lp cost of all crossing edges, we have
\begin{align*}
    \lp(E(v,T_u))&= \sum_{vw \in P} x_{vw} + \sum_{vw \in N} 1-x_{vw}\\
    &\geq  \sum_{vw \in P} x_{uv} - x_{uw} + \sum_{vw \in N} 1-x_{uv}-x_{uw}  \\
    &= px_{uv} + n(1-x_{uv}) - \sum_{w \in T_u} x_{uw} \\
    &> \rho p + \sigma n - \sigma(p + n)= (\rho - \sigma)p,
\end{align*}
where the first inequality follows from \ref{corr:lp:trieq}, and the final inequality uses $\rho \leq x_{uv} \leq 1 - \sigma$ and density of $T_u$ along with $|T_u| = |E(v,T_u)| = p + n$. As $ p = |\ep(v,T_u)|$ this concludes the proof.
\end{proof}}
For negative edges inside the cluster $T_u$, we do a similar case analysis based on whether the endpoints are close or far from $u$ (see accompanying diagram in \cref{fig:fig}). Let us start with the case where endpoints are close and hence the length of the negative edge is short by \ref{corr:lp:trieq}.
\begin{restatable}{lemma}{nonshortneg}\label{corr:lem:non-short-neg}
Let $N_s := \{vw \in \en(T_u): x_{uv},x_{uw} \leq \frac{\rho}{2}\}$, then $|N_s| \leq \frac{1}{1-\rho} \lp(N_s)$.
\end{restatable}
\full{\begin{proof}
Fix the edge $vw \in N_s$. Using (\ref{corr:lp:trieq}), $x_{vw} \leq x_{uv} + x_{uw} \leq \rho$ thus $1-x_{vw} \geq 1-\rho$ and so the cost of the edge can be charged to the \lp cost of the edge. 
\end{proof}}
To complete the cost analysis of the edges incident to $T_u$, it remains to bound the cost of negative intra-cluster edges where at least one endpoint is far from $u$. In order to not overcharge any edge, we need to fix an ordering on the vertices of $T_u$. For $v,w \in T_u$, define $v < w$ if $x_{uv} \leq x_{uw}$ (break ties consistently) and $T_u^<(v) = \{w \in T_u : w < v\}$. Again the proof relies on the density of $T_u$ and a counting argument.
\begin{restatable}{lemma}{nonlongneg}\label{corr:lem:non-long-neg}
For any $v \in T_u$ with $\rho/2 < x_{uv}$, $|\en(v, T_u^<(v))| \leq \frac{1}{1-\rho-\sigma} \lp(E(v, T_u^<(v)))$.
\end{restatable}
\full{\begin{proof} 
Let $p$ and $n$ be the number of vertices before $v$ in the ordering with positive and negative edge to $v$, respectively, i.e., $p = |\ep(v,T_u^<(v))|$ and let $n = \en(v,T_u^<(v))$.
\begin{align}
    \lp(E(v, T_u^<(v))) &= \sum_{\substack{vw \in \ep(v, T_u^<(v))}} x_{vw} + \sum_{\substack{vw \in \en(v, T_u^<(v))}} (1-x_{vw})& \notag\\
    &\geq \sum_{\substack{vw \in \ep(v, T_u^<(v))}} (x_{uv} - x_{uw})+ \sum_{\substack{vw \in \en(v, T_u^<(v))}} (1-x_{uv} - x_{uw})&  \text{(using (\ref{corr:lp:trieq}))}\notag\\
    & = p x_{uv} + n (1 - x_{uv}) - \sum_{w \in T_u^<(v)} x_{uw} &\notag\\
    & > \frac{\rho}{2}p + (1-\rho)n - \sum_{w \in T_u^<(v)} x_{uw}, &\label{corr:eq:non-long-neg:sub}
\end{align}
Where the last inequality is by the given assumption $ \frac{\rho}{2} < x_{uv} \leq \rho$. To bound the last term, we use the fact that $T_u$ is dense (i.e. $\sum_{w \in T_u} x_{uw} \leq \sigma |T_u|$), as follows.
\begin{align*}
    \sum_{w \in T_u} x_{uw} &\leq \sigma |T_u| \Rightarrow&\\
    \sum_{w \in T_u^<(v)} x_{uw} &\leq \sigma |T_u| -\sum_{\substack{w \in T_u: \\ w \geq v}} x_{uw}\\
    & = \sigma (p + n + |\{w \in T_u: w \geq v\}| ) -\sum_{\substack{w \in T_u: \\ w \geq v}} x_{uw}\\
    & < \sigma (p + n + |\{w \in T_u: w \geq v\}| ) - \frac{\rho}{2}|\{w \in T_u: w \geq v\}|&\text{(definition of ordering)}\\&&\text{ and $x_{uv} > \rho/2$)}\\
     &\leq \sigma (p + n). &\text{(since $\sigma \leq \rho/2$)}
\end{align*}

Next, we substitute this into \Cref{corr:eq:non-long-neg:sub}.

\begin{align*}
    \lp(E(v, T_u^<(v))) &> \frac{\rho}{2}p + (1-\rho)n - \sum_{w \in T_u^<(v)} x_{uw}\\
    & > \frac{\rho}{2}p + (1-\rho)n - \sigma(p + n) \geq (1-\rho-\sigma)n,
\end{align*}
since $\sigma \leq \rho/2$.
\end{proof}}
\full{This concludes the analysis of edges incident with non-degenerate clusters and enables the following proof.
\begin{proof} [Proof of \cref{corr:thm:non-degenerate}.]
First, for any internal edge $vw \in E(C)$ for some $C:= T_u \in \cC$ ($u$ may be the same as $v$ or $w$), it gets charged either through \cref{corr:lem:non-short-neg} or \cref{corr:lem:non-long-neg}. Based on the length of $x_{uv}$ and $x_{vw}$ and the defined ordering on vertices in $T_u$, $\lp(vw)$ is charged at most by one of these lemmas and so it is charged at most by $\max\{\frac{1}{1-\rho}, \frac{1}{1-\rho-\sigma}\} = \frac{1}{1-\rho-\sigma}$ as $\sigma > 0$.

For any crossing edge $vw \in E(C,V\backslash C)$  with $v \notin C$ and $C := T_u \in \cC$ ($u$ may be the same as $v$), it can gets charged either through \cref{corr:lem:non-long-pos} or \cref{corr:lem:non-short-pos}. Based on the length of $x_{uv}$, $\lp(vw)$ is charged at most by one of these lemmas and so it is charged at most $\max\{\frac{1}{1-\rho-\sigma}, \frac{1}{\rho-\sigma}\}$.
\end{proof} }

\subsection{Degenerate Clusters}\label{corr:sec:degen}
In this section, we focus on proving the following theorem for the cost of degenerate clusters. Here, the only disagreement edges are positive edges to other clusters and since we already counted the inter-cluster edges to non-degenerate clusters (\Cref{corr:lem:non-long-pos} and \Cref{corr:lem:non-long-neg}), we just need to focus on bounding the cost of positive edges between degenerate clusters with respect to the cost that \lp pays.

\begin{theorem}\label{corr:thm:degenerate}
For $\cC$ output of \Cref{corr:alg:main}, the correlation cost of the set of edges between degenerate clusters, i.e., $F = \bigcup_{u,u' \in \cC^1} \{uu'\}$ is at most 
\begin{eqnarray*}
\big[\max\{\frac{1}{\rho} + \frac{1-\rho}{\epsilon\rho}, \frac{1}{\sigma}, \frac{1}{2\sigma} + \frac{1}{2\epsilon\alpha^*}, \frac{1}{\epsilon\alpha^*}\}\big]\lp(F) + 
\frac{1}{2\epsilon}\cdot\frac{1-\rho}{\rho} \lp(E(\cC^1, V\backslash\cC^1)), 
\end{eqnarray*}
where $\alpha^* := \max_{i \in [\ell]} (1-\alpha_i)/\alpha_i$.
\end{theorem}

Recall in \Cref{corr:alg:main}, degenerate cluster $\{u\}$ is added to $\cC$ when at least one of the conditions in the if statement in Line~\ref{corr:ln:if} is not satisfied. At this time, all degenerate clusters, i.e., $\bigcup_{u\in\cC^1} \{u\}$ are added to $\cC$. Note that, unlike non-degenerate clusters, here $T_u$ and $T_{u'}$ for degenerate clusters $\{u\},\{u'\} \in \cC$ may overlap. 

Since for a positive edge $uv$ where $v\notin T_u$, we can apply \cref{corr:fact:long-pos}, our focus throughout this section is on edges inside $T_u$'s, starting with sparse $T_u$'s (\Cref{corr:lem:degen-sparse}\conf{, with proof in \Cref{corr:subsec:missing}}) and then moving to dense $T_u$'s. We conclude by using these lemmas to prove \cref{corr:thm:non-degenerate}. 
\full{The following facts are ubiquitously used in the proofs.}
\begin{fact}\label{corr:fact:long-pos}
The correlation clustering cost of any $(u, v) \in \ep$, for any $\rho > 0$, is at most $\lp(uv)/\rho$ if $x_{uv} > \rho$. 
\end{fact}
\full{\begin{fact}\label{corr:fact:sumxuv}
For any $F\subseteq E(u,T_u)$, $\sum_{uv \in F} x_{uv} \leq \lp(F)$.
\end{fact}
\begin{proof}
Take any $uv \in F$. If $uv \in \ep$ then $x_{uv} \leq \lp(uv)$. Now if $uv \in \en$, since $x_{uv} \leq \rho \leq 0.5$, $x_{uv} \leq 1-x_{uv} = \lp(uv)$.
\end{proof}}
\begin{restatable}{lemma}{degensparse}\label{corr:lem:degen-sparse}
For a degenerate cluster $\{u\} \in \cC$ with a sparse $T_u$, i.e., $(\sum_{v \in T_u} x_{uv})/|T_u| >\sigma$, 
\begin{equation}\label{corr:eq:degen-sparse}
    |T_u| \leq \frac{1}{\sigma}\lp(E(v,T_u)).
\end{equation}
\end{restatable}
\full{\begin{proof}
This just follows from rearranging density definition to $|T_u| \leq \sigma \sum_{uv \in E(u,T_u)} x_uv$ and applying \Cref{corr:fact:sumxuv}.
\end{proof}}

It remains to bound the cost of edges in $E(u,T_u)$ for a degenerate cluster ${u}$ with dense $T_u$. The following lemma is our main technical contribution.

\begin{lemma}\label{corr:lem:degen-dense}
For a degenerate cluster $\{u\} \in \cC$ with a dense $T_u$, i.e., $(\sum_{v \in T_u} x_{uv})/|T_u| \leq \sigma$,
{\small
\begin{equation}\label{corr:eq:degen-dense}
   |T_u|  \leq \frac{1}{\epsilon}[\frac{1}{\alpha^*} \lp(F_1)+ \frac{1-\rho}{\rho}\lp(F_2) +(1 + \frac{1}{\rho - \sigma}) \lp(F_3)],
\end{equation}
}
where $F_1 = E(u, T_u)$, $F_2 = E(u, \cC^1\backslash T_u)$, $F_3 = E(u, V\backslash \cC^1)$, and $\alpha^* := \max_{i \in [\ell]} (1-\alpha_i)/\alpha_i$.

\end{lemma}
\begin{proof}
Let  $i \in [\ell]$ be the color that violates the fairness constraint, i.e., $|V_i \cap T_u| > (1+\epsilon) \alpha_i |T_u|$. Note, this $i$ exists, since $u$ is a degenerate cluster (the if condition in \Cref{corr:alg:main} is false for $u$) but $T_u$ is dense. Using these facts in addition to the LP fairness constraints gives \Cref{corr:clm:newclaim} \conf{(proof in \Cref{corr:subsec:missing})}. Here, $\alpha_i' = \frac{\alpha_i}{1-\alpha_i}$. \conf{The last step is to plug in \Cref{corr:clm:claim}, which has a rather involved proof deferred to \Cref{corr:subsec:missing}.}
\begin{restatable}{claim}{newclaiminlemma}\label{corr:clm:newclaim}
$\epsilon \alpha_i'|T_u| \leq \lp(V_i\cap T_u)+ \alpha_i'\sum_{v \in (V\backslash T_u)\backslash V_i} (1-x_{uv})$.
\end{restatable}
\full{\begin{proof}
Since $x$ is a feasible solution to \fcclp: 
\begin{align*}
    \sum_{v \in V_i} (1-x_{uv}) &\leq \alpha_i\sum_{v \in V} (1-x_{uv}) \Rightarrow
    \sum_{v \in V_i} (1-x_{uv}) & \leq \frac{\alpha_i}{1-\alpha_i}\sum_{v \in V\backslash V_i} (1-x_{uv}).
\end{align*}
For ease of notation, let us refer to $\frac{\alpha_i}{1-\alpha_i}$ as $\alpha_i'$ so we have
\begin{align*}
\sum_{v \in V_i} (1-x_{uv}) & \leq \alpha_i'\sum_{v \in V\backslash V_i} (1-x_{uv})\Rightarrow\\
\sum_{v \in V_i\cap T_u} (1-x_{uv}) & \leq \alpha_i'\sum_{v \in V\backslash V_i} (1-x_{uv})\Rightarrow\\
|V_i\cap T_u|-\sum_{v \in V_i\cap T_u}x_{uv}  &\leq \alpha_i'\sum_{v \in V\backslash V_i} (1-x_{uv})\\
 & = \alpha_i'\sum_{v \in T_u\backslash V_i} (1-x_{uv}) + \alpha_i'\sum_{v \in (V\backslash T_u)\backslash V_i} (1-x_{uv})\\
& \leq \alpha_i' |T_u\backslash V_i|+ \alpha_i'\sum_{v \in (V\backslash T_u)\backslash V_i} (1-x_{uv})\Rightarrow\\
|V_i\cap T_u|-\alpha_i' |T_u\backslash V_i|  &\leq \sum_{v \in V_i\cap T_u}x_{uv} + \alpha_i'\sum_{v \in (V\backslash T_u)\backslash V_i} (1-x_{uv})\\
  &\leq_{\text{By \cref{corr:fact:sumxuv}}} \lp(V_i\cap T_u)+ \alpha_i'\sum_{v \in (V\backslash T_u)\backslash V_i} (1-x_{uv}).
\end{align*}
Using the fact that fairness constraint is violated for $V_i$, meaning, $|V_i \cap T_u| > (1+\epsilon) \alpha_i |T_u|$ and consequently $|T_u\backslash V_i| < (1 - (1+\epsilon)\alpha_i) |T_u|$ we have
\begin{equation*}
    (1+\epsilon) \alpha_i |T_u| - \alpha_i'(1 - (1+\epsilon) \alpha_i) |T_u| \leq \lp(V_i\cap T_u)+ \alpha_i'\sum_{v \in (V\backslash T_u)\backslash V_i} (1-x_{uv}). 
\end{equation*}
Recall, $\alpha_i' = \alpha_i/(1-\alpha_i)$ so the coefficient of $|T_u|$ on the LHS can be simplified to $\epsilon \alpha_i'$. 
So we get
\begin{equation*}\label{corr:eq:bef_claim}
     \epsilon \alpha_i'|T_u| \leq \lp(V_i\cap T_u)+ \alpha_i'\sum_{v \in (V\backslash T_u)\backslash V_i} (1-x_{uv}).  
\end{equation*}
\end{proof}}
\full{The last step is to plug in \Cref{corr:clm:claim} into \Cref{corr:clm:newclaim} to get the lemma statement.}
\begin{restatable}{claim}{claiminlemma}\label{corr:clm:claim} $\sum_{v \in (V\backslash T_u)\backslash V_i} (1-x_{uv})$ is at most $\frac{(1-\rho)}{\rho}\lp((\cC^1\backslash T_u)\backslash V_i) + \big(\frac{1}{\rho - \sigma}+1\big) \lp(V\backslash \cC^1)$.
\end{restatable}
\full{\begin{proof}
We bound the sum by breaking it into two cases based on whether $v$ is in $\cC^1$ or not.

\noindent
\textbf{Case $v\in  (\cC^1\backslash T_u)\backslash V_i$.} since $v \notin T_u$, $x_{uv} > \rho$ and so $1-x_{uv} < (1-\rho) \leq \frac{(1-\rho)}{\rho} x_{uv}$, therefore,
\begin{align}
\sum_{v \in (\cC^1\backslash T_u)\backslash V_i} (1-x_{uv}) =
     &\sum_{\substack{v \in (\cC^1\backslash T_u)\backslash V_i:\\ v \in \ep}} (1-x_{uv}) + \sum_{\substack{v \in (\cC^1\backslash T_u)\backslash V_i:\\ v \in \en}} (1-x_{uv}) \notag\\
    &\leq  \frac{(1-\rho)}{\rho}\sum_{\substack{v \in (\cC^1\backslash T_u)\backslash V_i:\\ v \in \ep}} x_{uv} + \sum_{\substack{v \in (\cC^1\backslash T_u)\backslash V_i:\\ v \in \en}} (1-x_{uv}) \notag\\
    &\leq \frac{(1-\rho)}{\rho}\lp((\cC^1\backslash T_u)\backslash V_i)\label{corr:eq:rhs21}.
\end{align}
    \textbf {Case $v\in (V\backslash \cC^1)\backslash V_i$.} Note that, $V\backslash \cC^1$ is already partitioned into non-degenerate clusters in $\cC$. Just for the sake of this proof, for any $v\in (V\backslash \cC^1)$, let $C(v)$ denote the center of $v$'s cluster. That is, $v \in T_{C(v)} \in \cC$. Now using this new notation, let us divide the set $(V\backslash \cC^1)\backslash V_i$ into three parts. The first part, is simply members with negative edges to $u$, that is, $N = \{v \in (V\backslash \cC^1)\backslash V_i: (u,v) \in \en\}$. For $v \in (V\backslash \cC^1)$ with $(u,v) \in \ep$, it falls into either the second or third part, depending on $x_{uC(v)}$ where $C(v)$ is the center of $v$'s cluster. That is, the second part is defined as $P_l = \{v \in (V\backslash \cC^1)\backslash V_i: (u,v) \in \ep \text{ and } x_{uC(v)} \geq (1-\sigma)\}$ and the third part is $P_s = \{v \in (V\backslash \cC^1)\backslash V_i: (u,v) \in \ep \text{ and } x_{uC(v)} < (1-\sigma)\}$.
\begin{align*}
\sum_{v \in (V\backslash \cC^1)\backslash V_i} (1-x_{uv})&= \sum_{v \in N} (1-x_{uv}) + \sum_{v \in P_l}(1-x_{uv}) + \sum_{v \in P_s} (1-x_{uv})\notag\\
&\leq \lp(N) + |P_l| + |P_s|.
\end{align*}
    Take any non-degenerate cluster centered at a $w$, that is, $T_w \in \cC$. $T_w$ could intersect at most one of $P_l$ or $P_s$ based on $x_{uw}$, and depending on which, we use \Cref{corr:lem:non-long-pos,corr:lem:non-short-pos} from $w$'s perspective to bound $|P_l|$ and $|P_s|$ respectively.
\begin{align*}
\sum_{v \in (V\backslash \cC^1)\backslash V_i} (1-x_{uv})&\leq_{\text{(\Cref{corr:lem:non-long-pos})}} \lp(N) + \frac{1}{1-\rho-\sigma}\lp(P_l) + |P_s| &\notag\\
&\leq_{\text{(\Cref{corr:lem:non-short-pos})}} \lp(N) + \frac{1}{1-\rho-\sigma}\lp(P_l) + \frac{1}{\rho-\sigma} \sum_{\substack{C \in \cC:\\ C\cap P_s \neq \emptyset}} \lp(C)\notag\\
    &\leq \big(\frac{1}{\rho - \sigma}+1\big) \lp(V\backslash \cC^1).
\end{align*}

\end{proof}}
\full{This concludes the proof of the lemma.}
\end{proof}
\full{We now have the required ingredients for proving \cref{corr:thm:degenerate}.}

\begin{proof}[Proof of \cref{corr:thm:degenerate}]
This proof is obtained by combining three sets of inequalities: (i) inequalities for long positive edges by \Cref{corr:fact:long-pos}, (ii) inequalities in \cref{corr:lem:degen-sparse} for degenerate $u$ with sparse $T_u$, and (iii)  inequalities in \cref{corr:lem:degen-dense} for for degenerate $u$ with dense $T_u$. We use the size of $T_u$ as an upper-bound on $E(u,T_u)$. 

Take any two degenerates $u,v \in \cC^1$ with a short edge $uv$, i.e., $x_{uv} < \rho$. Note that in this case, $T_u$ and $T_v$ intersect so using either of the bounds in \Cref{corr:lem:degen-sparse} or \Cref{corr:lem:degen-dense} counts this edge twice, once from each of its endpoints. Therefore, we use a coefficient of $.5$ for \cref{corr:eq:degen-sparse}(of \Cref{corr:lem:degen-sparse}) and \cref{corr:eq:degen-dense}(of \Cref{corr:lem:degen-dense}). Putting all these together, each edge between degenerate clusters is counted exactly once on the left hand side and we only have edges incident to degenerate clusters on the right hand side, i.e., edge $uu'$ for $u, u' \in \cC^1$ or $uv$ for $u \in \cC^1, v \in V\backslash\cC^1$. So let us summarize how we bound the cost of an edge $uu'$ for $u, u \in \cC^1$, based on the length of the edge and condition of endpoints. All these terms can be bounded by the $\max$ term in the lemma:
{\small
\begin{equation*}
    \begin{cases}
    \text{       } x_{uu'} > \rho: \frac{1}{\rho} + \frac{1}{\epsilon}\frac{1-\rho}{\rho} &  \text{\cref{corr:fact:long-pos} \& Eq. \ref{corr:eq:degen-dense} for $u,u'$\footnotemark,}\\
    \text{$T_u \& T_{u'}$ sparse}:\frac{1}{\sigma}& \text{Eq. \ref{corr:eq:degen-sparse} for $u,u'$,} \\ 
    \text{$T_u \oplus T_{u'}$\footnotemark sparse}:\frac{1}{2\sigma} + \frac{1}{2\epsilon}\frac{1}{\alpha^*} & \text{Eq. \ref{corr:eq:degen-sparse} \& \ref{corr:eq:degen-dense} for $u,u'$,} \\ 
    \text{$T_u \& T_{u'}$ dense}:\frac{1}{\epsilon}\frac{1}{\alpha^*} & \text{Eq. \ref{corr:eq:degen-dense} for $u,u'$,} 
    \end{cases}
\end{equation*} 
}
\addtocounter{footnote}{-1}
\footnotetext{This is in the worst case that both $T_u$ and $T_{u'}$ are dense.}
\stepcounter{footnote}
\footnotetext{This symbol denotes exclusive or, meaning that exactly one of $T_u$ and $T_{u'}$ is sparse.}
\noindent

 For an edge $uv$ for $u \in \cC^1 ,v \in V\backslash\cC^1$, it can only get charged if $T_u$ is dense and since we only consider $.5$ of Equation \ref{corr:eq:degen-dense}, it gets charged at most $\frac{1}{2\epsilon}\cdot\frac{1-\rho}{\rho}$. Combining results for the discussed two cases, we get the statement of the lemma. 
\end{proof}

\subsection{Proof of \texorpdfstring{\Cref{corr:thm:main}}{Theorem}}
\full{Now we have all the tools necessary to prove the main theorem of this chapter.}
\begin{proof}[proof of \Cref{corr:thm:main}]
For the choice of $\rho = .5$ and $\sigma = .25$, we get that an edge incident to an endpoint in a non-degenerate cluster is charged at most $4$ by \cref{corr:thm:non-degenerate} and $\frac{1}{2\epsilon}$ by \cref{corr:thm:degenerate}, so the total of $4 + \frac{1}{2\epsilon}$. For an edge between degenerate clusters, it is charged by at most 
{\small
$$
\max\{2+ \frac{1}{\epsilon}, 4, 2 + \frac{1}{2\epsilon\alpha^*}, \frac{1}{\epsilon\alpha^*}\} \leq \max\{4 + \frac{1}{\epsilon}, \frac{1}{\epsilon\alpha^*}\}.
$$}
So an edge is charged at most $\max\{\frac{1}{\epsilon\alpha^*}, 4 + \frac{1}{\epsilon}\}$.
\end{proof}

\section{Experiments}\label{corr:sec:experiments}
In this section, we present the results of our experiments, designed to measure the quality of our algorithm. Our key findings are: \textbf{(1)} Our clustering cost is considerably better than our proven approximation ratio, namely, at most $15\%$ more than the optimal cost even for $\epsilon = 0.01$.
\textbf{(2)} The maximum fairness violation of our algorithm on non-singleton clusters is often much less than $\epsilon$. On some datasets, with varying $\epsilon$, violation is fixed on a small constant as early as $\epsilon = 0.3$.\full{ Our codes are publicly available on GitHub.\footnote{\url{github.com/moonin12/improved_fair_correlation_clustering}}} We start the section by describing our datasets, followed by quality measures and benchmarks.

\textbf{Datasets.} We use the datasets from~~\cite{ahmadian2020fair}: \amazon~\cite{leskovec2007dynamics} dataset publicly available on SNAP\footnote{\url{snap.stanford.edu/data/}} that corresponds to 2,441,053 items on Amazon with $+$ edges between co-reviewed items and $-$ edges for the rest. We also use datasets publicly available on the UCI repository\footnote{\url{archive.ics.uci.edu/ml/datasets/}}, \reuters~\cite{reuters_cite} and \victorian~\cite{gungor2018fifty}\footnote{The datasets are available at \url{ archive.ics.uci.edu/ml/datasets/Reuter_50_50} and \url{archive.ics.uci.edu/ml/datasets/Victorian+Era+Authorship+Attribution}} text data corresponding to up to 16 authors, 50 to 100 texts for each. The text is embedded into a 10 dimensional space using Gensim’s Doc2Vec~\cite{Gensim} then set the top $\theta = \{0.25, 0.5, 0.75\}$ fraction of edges in terms of cosine similarity as positive. See~\Cref{corr:subsec:dataset} for additional datasets with overlapping colors.

\textbf{Quality Measures.} We report \emph{cost ratio} which is the ratio of clustering cost to $|E|$. For the LP, this is the LP objective divided by $|E|$. For fairness analysis, we measure \emph{maximum fairness violation} defined as $\max_{C \in \cC, i \in [\ell]} |V_i \cap C|/(\alpha_i|C|) - 1 $ of non-degenerate clusters (note, this is bounded by $\epsilon$).

\textbf{Benchmarks.} We compare with the fair clustering algorithm of~\cite{ahmadian2020fair} as well as results they report on two fairness-oblivious \correlation algorithms Loc, their internal local search algorithm for \correlation, as well as Piv~\cite{acn}. Since their reported cost ratios are relative to problem size, we use the numbers without implementing their algorithm. As for~\cite{ahmadi2020fair} published on arXiv, their code is not available and their plots are not re-usable in this fashion. 

To compare with~\cite{ahmadian2020fair}, we set $\alpha$'s uniformly to $1/\ell$. To experiment on datasets with overlapping colors, used in~\cite{ahmadi2020fair,BCFN19, CKLV17}, we set $\alpha$'s uniformly to $0.8$ in accordance with the disparate impact doctrine~\cite{EEOC}.
To tune the parameters $\rho$ and $\sigma$, for each dataset, we try 5 different values for $\rho$ from $0.1$ to $0.5$, and 10 different values for $\sigma$ from $0.1\times \rho/2$ to $\rho/2$. Also, we shuffle the points 20 times, re-run the algorithm, and report the result with the best clustering cost. We solve the \lp using CPLEX~\cite{CPLEX21}.


\subsection{Cost Analysis}
We compare our clustering cost to the \lp objective which is a \emph{lower bound} on the cost of the optimal solution. Our results demonstrate that we are much closer to the optimal solution than our theoretical results suggest. We experiment once with varying $\epsilon$ and once with varying $\alpha^{\min} := \min_{i \in [\ell]} \alpha_i$, scaling all others $\alpha$s accordingly (see \Cref{corr:fig:varepsminalpha}). \Cref{corr:fig:varepsminalpha} suggests that allowing the algorithm to violate fairness, our cost can even beat the optimal cost (which is constrained by fairness). As stated in \Cref{corr:thm:main}, the approximation ratio of our algorithm drops significantly by increasing $\epsilon$. \Cref{corr:fig:varepsminalpha} depicts how the LP costs drop by relaxing fairness through increasing $\alpha^{\min}$ and that our cost almost matches the optimal cost even for $\epsilon = 0.01$. See \Cref{corr:subsec:appendcost} for the full set of experiments.

\begin{figure}[ht]
\begin{center}
\centering
\includegraphics[width=0.4\columnwidth]{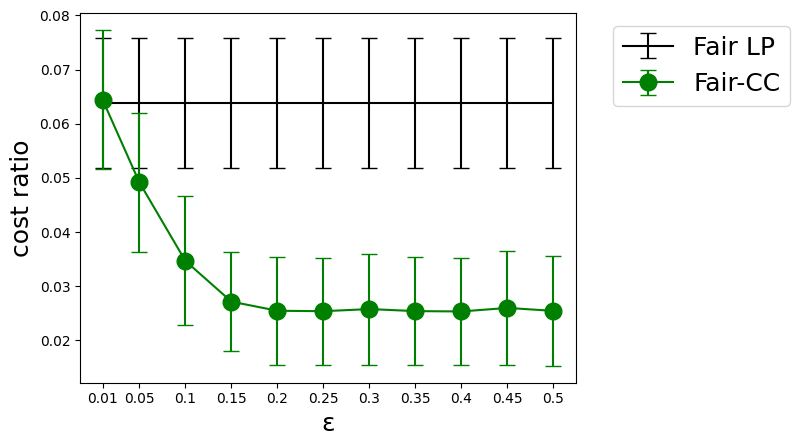}
\includegraphics[width=0.4\columnwidth]{plots/var eps/amazon_theta-1_colors0-1_samp9_alphas1_scale0_lp_var_eps.png}
\caption{Cost ratio of our algorithm (Fair-CC) and the fair \lp for \amazon, 10 sub-samples of 200 each. The first plot uses varying $\epsilon$ from $0.01$, to $0.5$. The second plot is for $\epsilon = 0.01$, varying $\alpha^{\min}$ from its original value $0.5$ to $1$ (no fairness), scaling other $\alpha$'s accordingly.}
\label{corr:fig:varepsminalpha}
\end{center}
\end{figure}



Furthermore, we compare our cost ratio with that of ~\cite{ahmadian2020fair}. The \lp solutions allows us to compare with optimum, but in contrast, ~\cite{ahmadian2020fair} do not have any proxy for the optimal cost and just compare their cost ratio against popular fairness-oblivious correlation clusterings, i.e., Piv and Loc. For the sake of completeness, we also compare our results to that of \cite{ahmadian2020fair} and algorithms therein. \Cref{corr:tab:prevcomp} demonstrates that even for $\epsilon$ as small as $0.01$, our cost ratios are considerably better.


\begin{table}[t]
\caption{Cost ratio comparison for datasets in~\cite{ahmadian2020fair} (AKEM) for the case of two colors, $\alpha_1 = \alpha_2 = .5$, $\epsilon = 0.01$. \amazon is the average and standard deviation reported on 20 sub-samples of size 200. Loc and Piv are \correlation algorithms demonstrated to be unfair on these datasets~\cite{ahmadian2020fair}}
\label{corr:tab:prevcomp}
\begin{center}
\begin{scriptsize}
\begin{sc}
\begin{tabular}{lcccr}
\toprule
        &  \multicolumn{2}{c}{\makecell{Fair Algorithms}}  &  \multicolumn{2}{c}{\makecell{Unfair Algorithms}}    \\
        \\
        Dataset&  \textbf{Fair-CC} & AEKM & Loc & Piv  \\
\midrule
\amazon    & $0.064 \pm 0.013$ & 0.064& 0.010 & 0.011\\
\reuters $\theta = 0.25$  &{\bf0.213}& 0.230 & 0.096& 0.161\\
\reuters $\theta = 0.50$ &{\bf0.297}& 0.350 & 0.181& 0.231\\
\reuters $\theta = 0.75$ &{\bf0.196}& 0.199 & 0.188& 0.241\\
\victorian $\theta = 0.25$&0.217& {\bf0.212} & 0.109& 0.158\\
\victorian $\theta = 0.50$&{\bf0.325}& 0.348 & 0.183& 0.268\\
\victorian $\theta = 0.75$&{\bf0.232}& 0.237 & 0.203& 0.280\\
\bottomrule
\end{tabular}
\end{sc}
\end{scriptsize}
\end{center}
\end{table}
\subsection{Fairness Analysis}
In this section, we report result of study on how relaxing the allowed violation in fairness constraints affects the actual final violation of fairness in our output clusters. \Cref{corr:fig:fairvarepsgood} shows a case where maximum fairness violation of our algorithm is much less than $\epsilon$ (max allowed violation) and bounded by $0.11$. On some datasets like \amazon in \Cref{corr:fig:fairvarepsgood}, the maximum violation curve demonstrates a few \emph{elbows} for varying $\epsilon$. This itself is of independent interest and may give more insights on dataset upon closer inspection. See \Cref{corr:subsec:appendfair} for a full set of experiments.
\begin{figure}[ht]
\begin{center}
\centering
\includegraphics[width=0.4\columnwidth]{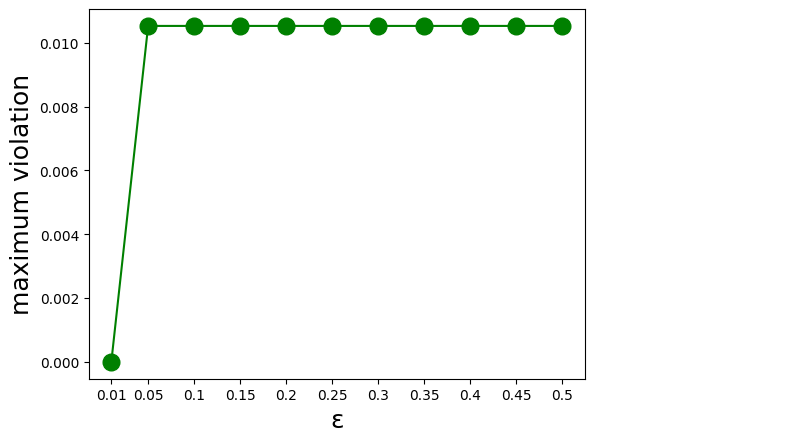}
\includegraphics[width=0.4\columnwidth]{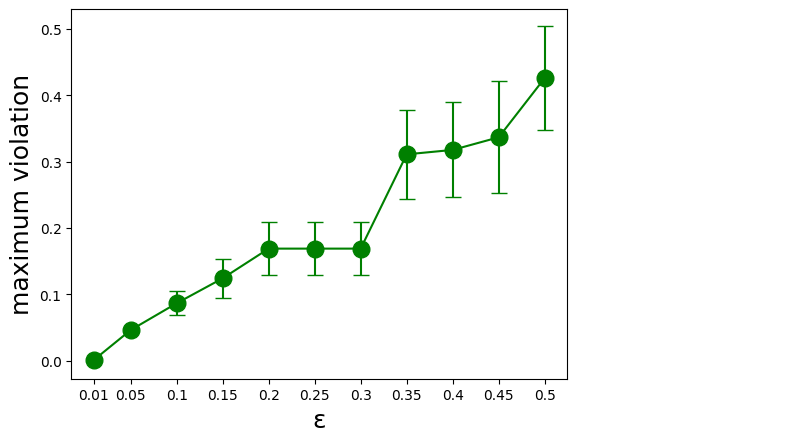}
\caption{fairness violation of our algorithm (Fair-CC), varying $\epsilon$ from $0.01$ to $0.5$ on \reuters $\theta=0.75$ and \amazon.}
\label{corr:fig:fairvarepsgood}
\end{center}
\end{figure}

\clearpage
\full{\section{Acknowledgements}
A preliminary version of this work appeared in Maryam Negahbani's Ph.D. thesis. We would like to thank Deeparnab Chakrabarty for valuable feedback that greatly improved the representation of this paper. Maryam Negahbani was funded on NSF Award \#2041920.}
\bibliographystyle{plain}
\bibliography{main}{}
\conf{\input{neurips_checklist}}
\newpage
\appendix
\conf{\section{Missing Proofs from \Cref{corr:sec:analysis}}\label{corr:subsec:missing}
In this section, we restate Claims and Lemmas from \Cref{corr:sec:analysis} for which the proof is deferred to this section.
\nonlongpos*
\begin{proof}
Fix an edge $vw \in E^+(v, T_u)$ ($w$ may or may not be equal to $u$). Roughly speaking, since $v$ is far from $u$ and $w$ is close to $u$ ($w$ belongs to $T_u$), the length of edge $vw$ cannot be too short. Formally, by \ref{corr:lp:trieq}, 
\begin{align}
x_{vw} \geq x_{uv} - x_{uw} \geq 1 - \sigma - \rho &\tag*{since  $x_{uw} \leq \rho$,} 
\end{align}
and so the cost of this edge can be just charged to the \lp cost the edge. 
\end{proof}
\nonshortpos*
\begin{proof}
Let $P := \ep(v,T_u)$ with $p = |P|$ and let $N := \en(v,T_u)$ with $n = |N|$. Analyzing the \lp cost of all crossing edges, we have
\begin{align*}
    \lp(E(v,T_u))&= \sum_{vw \in P} x_{vw} + \sum_{vw \in N} 1-x_{vw}\\
    &\geq  \sum_{vw \in P} x_{uv} - x_{uw} + \sum_{vw \in N} 1-x_{uv}-x_{uw}  \\
    &= px_{uv} + n(1-x_{uv}) - \sum_{w \in T_u} x_{uw} \\
    &> \rho p + \sigma n - \sigma(p + n)= (\rho - \sigma)p,
\end{align*}
where the first inequality follows from \ref{corr:lp:trieq}, and the final inequality uses $\rho \leq x_{uv} \leq 1 - \sigma$ and density of $T_u$ along with $|T_u| = |E(v,T_u)| = p + n$. As $ p = |\ep(v,T_u)|$ this concludes the proof.
\end{proof}
\nonshortneg*
\begin{proof}
Fix the edge $vw \in N_s$. Using (\ref{corr:lp:trieq}), $x_{vw} \leq x_{uv} + x_{uw} \leq \rho$ thus $1-x_{vw} \geq 1-\rho$ and so the cost of the edge can be charged to the \lp cost of the edge. 
\end{proof}
\nonlongneg*
\begin{proof} 
Let $p$ and $n$ be the number of vertices before $v$ in the ordering with positive and negative edge to $v$, respectively, i.e., $p = |\ep(v,T_u^<(v))|$ and let $n = \en(v,T_u^<(v))$.
\begin{align}
    \lp(E(v, T_u^<(v))) &= \sum_{\substack{vw \in \ep(v, T_u^<(v))}} x_{vw} + \sum_{\substack{vw \in \en(v, T_u^<(v))}} (1-x_{vw})& \notag\\
    &\geq \sum_{\substack{vw \in \ep(v, T_u^<(v))}} (x_{uv} - x_{uw})+ \sum_{\substack{vw \in \en(v, T_u^<(v))}} (1-x_{uv} - x_{uw})&  \text{(using (\ref{corr:lp:trieq}))}\notag\\
    & = p x_{uv} + n (1 - x_{uv}) - \sum_{w \in T_u^<(v)} x_{uw} &\notag\\
    & > \frac{\rho}{2}p + (1-\rho)n - \sum_{w \in T_u^<(v)} x_{uw}, &\label{corr:eq:non-long-neg:sub}
\end{align}
Where the last inequality is by the given assumption $ \frac{\rho}{2} < x_{uv} \leq \rho$. To bound the last term, we use the fact that $T_u$ is dense (i.e. $\sum_{w \in T_u} x_{uw} \leq \sigma |T_u|$), as follows.
\begin{align*}
    \sum_{w \in T_u} x_{uw} &\leq \sigma |T_u| \Rightarrow&\\
    \sum_{w \in T_u^<(v)} x_{uw} &\leq \sigma |T_u| -\sum_{\substack{w \in T_u: \\ w \geq v}} x_{uw}\\
    & = \sigma (p + n + |\{w \in T_u: w \geq v\}| ) -\sum_{\substack{w \in T_u: \\ w \geq v}} x_{uw}\\
    & < \sigma (p + n + |\{w \in T_u: w \geq v\}| ) - \frac{\rho}{2}|\{w \in T_u: w \geq v\}|&\text{(definition of ordering)}\\&&\text{ and $x_{uv} > \rho/2$)}\\
     &\leq \sigma (p + n). &\text{(since $\sigma \leq \rho/2$)}
\end{align*}

Next, we substitute this into \Cref{corr:eq:non-long-neg:sub}.

\begin{align*}
    \lp(E(v, T_u^<(v))) &> \frac{\rho}{2}p + (1-\rho)n - \sum_{w \in T_u^<(v)} x_{uw}\\
    & > \frac{\rho}{2}p + (1-\rho)n - \sigma(p + n) \geq (1-\rho-\sigma)n,
\end{align*}
since $\sigma \leq \rho/2$.
\end{proof}
This concludes the analysis of edges incident with non-degenerate clusters and enables the following proof.
\begin{proof} [Proof of \cref{corr:thm:non-degenerate}.]
First, for any internal edge $vw \in E(C)$ for some $C:= T_u \in \cC$ ($u$ may be the same as $v$ or $w$), it gets charged either through \cref{corr:lem:non-short-neg} or \cref{corr:lem:non-long-neg}. Based on the length of $x_{uv}$ and $x_{vw}$ and the defined ordering on vertices in $T_u$, $\lp(vw)$ is charged at most by one of these lemmas and so it is charged at most by $\max\{\frac{1}{1-\rho}, \frac{1}{1-\rho-\sigma}\} = \frac{1}{1-\rho-\sigma}$ as $\sigma > 0$.

For any crossing edge $vw \in E(C,V\backslash C)$  with $v \notin C$ and $C := T_u \in \cC$ ($u$ may be the same as $v$), it can gets charged either through \cref{corr:lem:non-long-pos} or \cref{corr:lem:non-short-pos}. Based on the length of $x_{uv}$, $\lp(vw)$ is charged at most by one of these lemmas and so it is charged at most $\max\{\frac{1}{1-\rho-\sigma}, \frac{1}{\rho-\sigma}\}$.
\end{proof} 
\degensparse*
\begin{proof}
This just follows from rearranging density definition to $|T_u| \leq \sigma \sum_{uv \in E(u,T_u)} x_uv$ and applying \Cref{corr:fact:sumxuv}.
\end{proof}
\begin{fact}\label{corr:fact:sumxuv}
For any $F\subseteq E(u,T_u)$, $\sum_{uv \in F} x_{uv} \leq \lp(F)$.
\end{fact}
\begin{proof}
Take any $uv \in F$. If $uv \in \ep$ then $x_{uv} \leq \lp(uv)$. Now if $uv \in \en$, since $x_{uv} \leq \rho \leq 0.5$, $x_{uv} \leq 1-x_{uv} = \lp(uv)$.
\end{proof}

The following two claims are defined as part of the proof of \cref{corr:lem:degen-dense}, so the reader should refer to the notation defined therein.
\newclaiminlemma*

\claiminlemma*
}
\section{Complimentary Experiments}
\subsection{Additional Comparison with Previous Work}
In this section, we bring our comparison results with previous work for 4 and 8 colors in \Cref{corr:tab:prevcompfour,corr:tab:prevcompeight} respectively.
\begin{table}[th!]
\caption{Cost ratio comparison for datasets in~\cite{ahmadian2020fair} (AKEM) for the case of four colors and all $\alpha$'s equal to $.25$. We use $\epsilon = 0.01$. \amazon and \victorian are sub-sampled to 200 stratified on color combinations. We report the average and standard deviation over 20 sub-samples. Loc is a \correlation algorithm demonstrated to be unfair on these datasets~\cite{ahmadian2020fair}.}\label{corr:tab:prevcompfour}
\begin{center}
\begin{scriptsize}
\begin{sc}
\begin{tabular}{lcccr}

\toprule
        &  \multicolumn{2}{c}{\makecell{Fair Algorithms}}  &  \multicolumn{1}{c}{\makecell{Unfair Algorithm}}    \\
        \\
        Dataset&  \textbf{Fair-CC} & AEKM & Loc \\
\midrule
\reuters $\theta = 0.25$  &0.250& {\bf0.244} & 0.120 \\
\reuters $\theta = 0.50$ &{\bf0.306}& 0.336 & 0.191 \\
\reuters $\theta = 0.75$ &{\bf0.218}& 0.227 & 0.211 \\
\victorian $\theta = 0.25$&$0.240 \pm 0.012$& {\bf0.210} & 0.141 \\
\victorian $\theta = 0.50$&$0.322 \pm 0.014$& {\bf0.311} & 0.228 \\
\victorian $\theta = 0.75$&{$\bf0.240 \pm 0.007$}& 0.245 & 0.225 \\
\bottomrule
\end{tabular}
\end{sc}
\end{scriptsize}
\end{center}
\end{table}

\begin{table}[th!]
\caption{Cost ratio comparison for datasets in~\cite{ahmadian2020fair} (AKEM) for the case of eight colors and all $\alpha$'s equal to $.125$. We use $\epsilon = 0.01$. Datasets are sub-sampled to 200 stratified on color combinations. We report the average and standard deviation over 20 sub-samples for each. Loc is a \correlation algorithm demonstrated to be unfair on these datasets~\cite{ahmadian2020fair}.}
\label{corr:tab:prevcompeight}
\begin{center}
\begin{scriptsize}
\begin{sc}
\begin{tabular}{lcccr}
\toprule
        &  \multicolumn{2}{c}{\makecell{Fair Algorithms}}  &  \multicolumn{1}{c}{\makecell{Unfair Algorithms}}    \\
        \\
        Dataset&  \textbf{Fair-CC} & AEKM & Loc\\
\midrule
\reuters $\theta = 0.25$  &${\bf0.250 \pm 0.000}$& 0.252 & 0.133\\
\reuters $\theta = 0.50$ &${0.5 \pm 0.000}$& {\bf0.426} & 0.239\\
\reuters $\theta = 0.75$ &${\bf0.244 \pm 0.003}$& 0.250 & 0.237\\
\victorian $\theta = 0.25$&${0.250 \pm 0.000}$& {\bf0.212} & 0.161\\
\victorian $\theta = 0.50$&${0.370 \pm 0.085}$& {\bf0.319} & 0.249\\
\victorian $\theta = 0.75$&${\bf0.242 \pm 0.006}$& 0.246 & 0.218\\
\bottomrule
\end{tabular}
\end{sc}
\end{scriptsize}
\end{center}
\end{table}
\subsection{Experiments with Overlapping Colors }\label{corr:subsec:dataset}
We use datasets publicly available on the UCI repository\footnote{\url{archive.ics.uci.edu/ml/datasets/}}\textbf{(1)} \bank~\cite{bank-paper} with 4,521 points, corresponding to phone calls from a marketing campaign by a
Portuguese banking institution.
\textbf{(2)} \census~\cite{census-paper} with 32,561 points, representing information about individuals extracted from the 1994 US census.
\textbf{(3)} \diabetes~\cite{strack2014impact} 
with 101,766 points, extracted from diabetes patient records.

\bank, \census, and \diabetes are used in ~\cite{ahmadian2019clustering,BCFN19, CKLV17} and have 5, 7, and 8 colors respectively where each node has exactly two colors \Cref{corr:table:data}. These are sub-sampled to 200 stratified on color combinations. Since the previous work does not handle overlapping colors, we only compare with the LP cost in \cref{corr:tab:compovercolor}

\begin{table}[!ht]
\centering
\caption{Detailed description of the datasets \bank, \census, and \diabetes. For each dataset, the coordinates are the numeric attributes used to determined the position of each record in the Euclidean space. The sensitive attributes determines protected groups.}\label{corr:table:data}
\begin{tabular}{ llrl }
\rule{0pt}{4ex}
\textbf{Dataset}& \textbf{Coordinates} & \textbf{\shortstack{Sensitive\\  attributes}}&\textbf{Protected groups} \\ 
 \hline 
 \rule{0pt}{2.5ex}
 \multirow{2}*\textbf{\bank}
 & age, balance, duration & marital & married, single, divorced \\
 \cmidrule(r){3-4}
 &  & default &  yes, no \\
 \hline 
 \rule{0pt}{2.5ex}
 \multirow{3}*\textbf{\census}
 & age, education-num,  & sex & female, male\\
 \cmidrule(r){3-4}
  &final-weight, capital-gain, & race & Amer-ind, asian-pac-isl,\\
  \rule{0pt}{2.5ex}
   & hours-per-week && black, other, white\\
   
    \hline 
 \rule{0pt}{2.5ex}
 \multirow{2}*\textbf{\diabetes}
 & gender, age, race, & gender & female, male \\
 \cmidrule(r){3-4}
 & time-in-hospital & race & 6 groups \\

   \hline
\end{tabular}

\end{table}

\begin{table}[th!]
\caption{Cost ratio comparison with the LP cost for datasets with overlapping colors, used in~\cite{ahmadian2019clustering, BCFN19, CKLV17}. all $\alpha$'s are set to $0.8$ in accordance with the DI doctrine~\cite{EEOC}. We use $\epsilon = 0.01$. Datasets are sub-sampled to 200 stratified on color combinations.}\label{corr:tab:compovercolor}
\begin{center}
\begin{scriptsize}
\begin{sc}
\begin{tabular}{lcc}
\toprule
        \\
        Dataset&  Fair-CC cost & Fair LP cost \\
\midrule

\bank $\theta = 0.25$  & 0.249& 0.248\\
\bank $\theta = 0.50$  & 0.498& 0.497\\
\bank $\theta = 0.75$  & 0.749& 0.746\\
\midrule
\census $\theta = 0.25$  & 0.135& 0.108\\
\census $\theta = 0.50$  & 0.226& 0.190\\
\census $\theta = 0.75$  & 0.268& 0.276\\
\midrule
\diabetes $\theta = 0.25$  & 0.100& 0.077\\
\diabetes $\theta = 0.50$  & 0.143& 0.122\\
\diabetes $\theta = 0.75$  & 0.130& 0.119\\
\bottomrule
\end{tabular}
\end{sc}
\end{scriptsize}
\end{center}
\end{table}
\clearpage
\subsection{Cost Analysis}\label{corr:subsec:appendcost}
In this section, we have the set of complete cost analysis experiments over all the datasets except \amazon, which can be found in \Cref{corr:sec:experiments}. For each dataset we have a pair of figures depicting the cost ratio of our algorithm (Fair-CC) and the fair \lp for that dataset. On the left, varying $\epsilon$ from $0.01$ (with completely fair clusters), to $0.5$. On the right, for $\epsilon = 0.01$ and varying $\alpha^{\min}$ from its original value $0.5$ to 1 (no fairness), scaling other $\alpha$'s accordingly.

\begin{figure}[!ht]
  \centering
    \includegraphics[width=0.45\columnwidth]{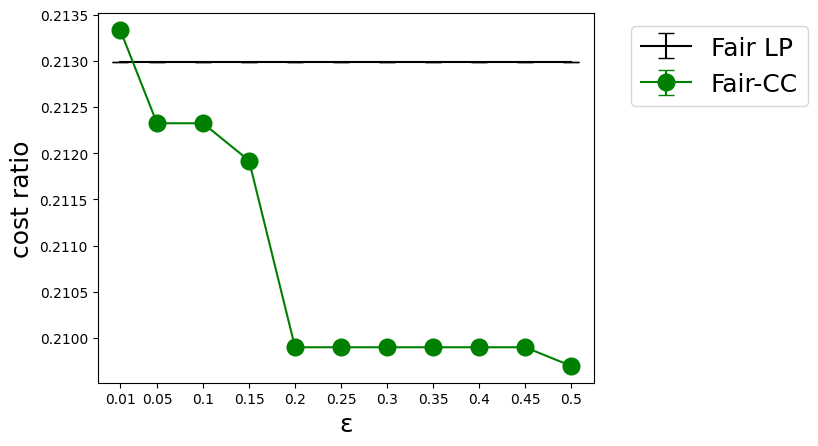}
    \includegraphics[width=0.45\columnwidth]{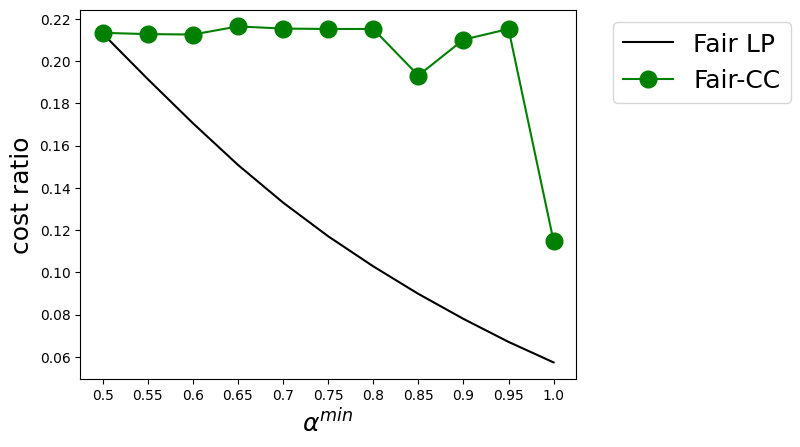}
  \caption{\reuters $\theta = 0.25$, cost ratios of our algorithm (fair-CC) and fair LP for varying $\epsilon$. }
\end{figure}

\begin{figure}[!ht]
  \centering
    \includegraphics[width=0.45\columnwidth]{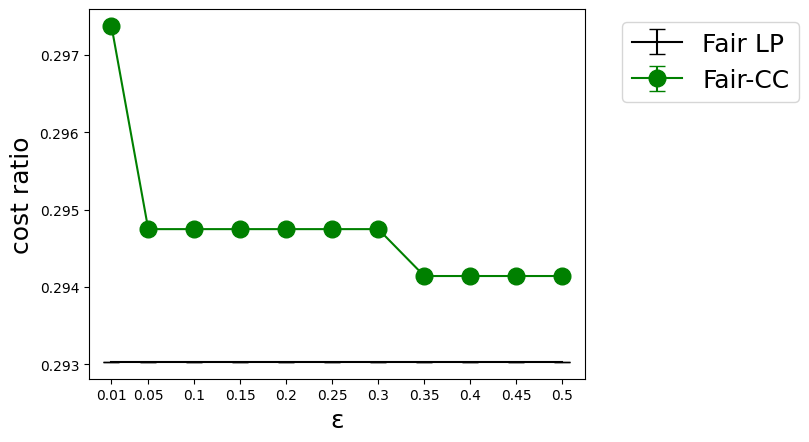}
    \includegraphics[width=0.45\columnwidth]{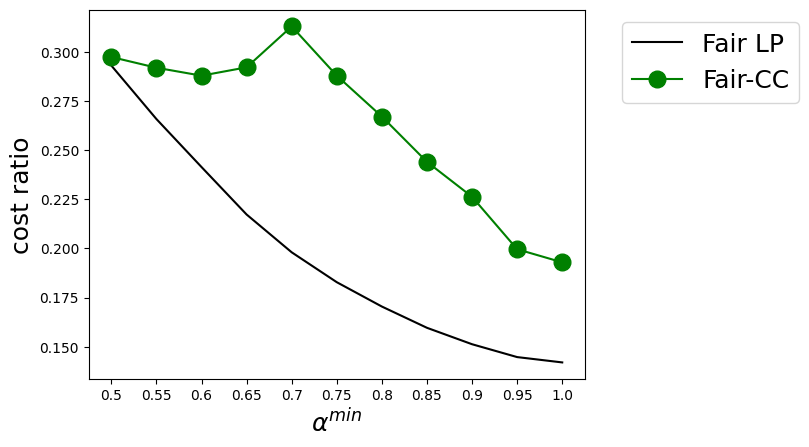}
  \caption{\reuters $\theta = 0.5$, cost ratios of our algorithm (fair-CC) and fair LP for varying $\epsilon$. }
\end{figure}

\begin{figure}[!ht]
  \centering
    \includegraphics[width=0.45\columnwidth]{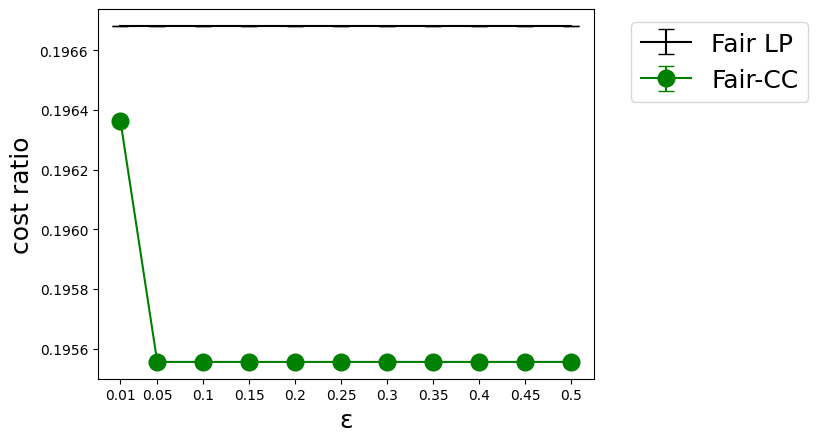}
    \includegraphics[width=0.45\columnwidth]{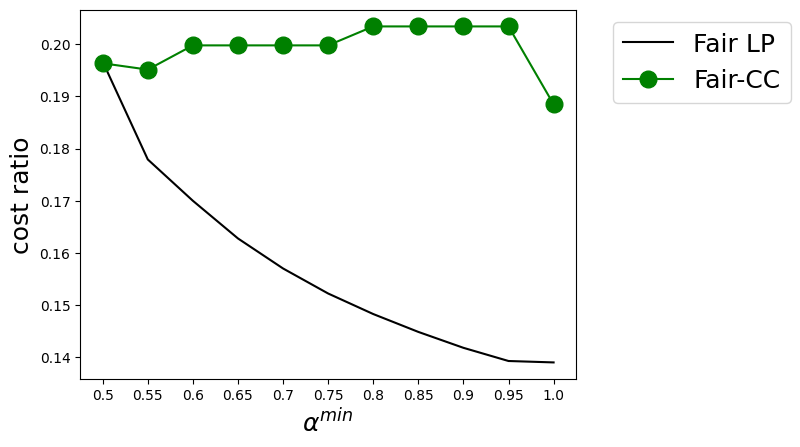}
  \caption{\reuters $\theta = 0.75$, cost ratios of our algorithm (fair-CC) and fair LP for varying $\epsilon$. }
\end{figure}

\begin{figure}[!ht]
  \centering
    \includegraphics[width=0.45\columnwidth]{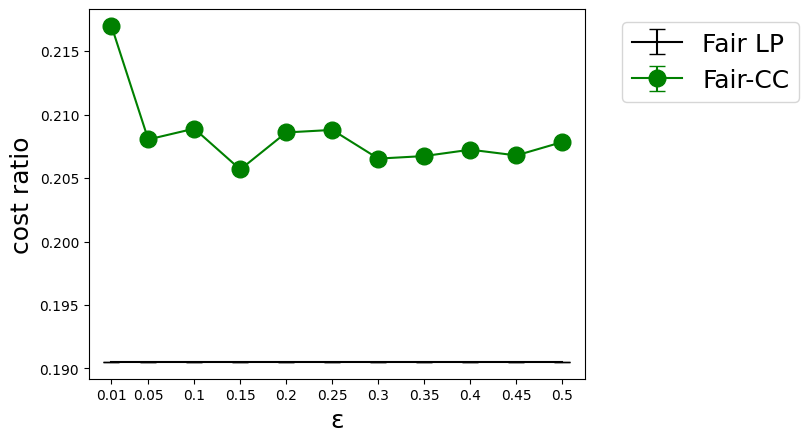}
    \includegraphics[width=0.45\columnwidth]{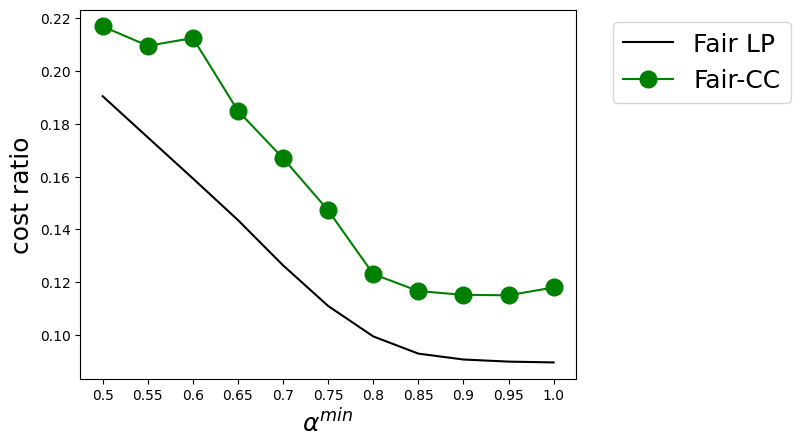}
  \caption{\victorian $\theta = 0.25$, cost ratios of our algorithm (fair-CC) and fair LP for varying $\epsilon$. }
\end{figure}

\begin{figure}[!ht]
  \centering
    \includegraphics[width=0.45\columnwidth]{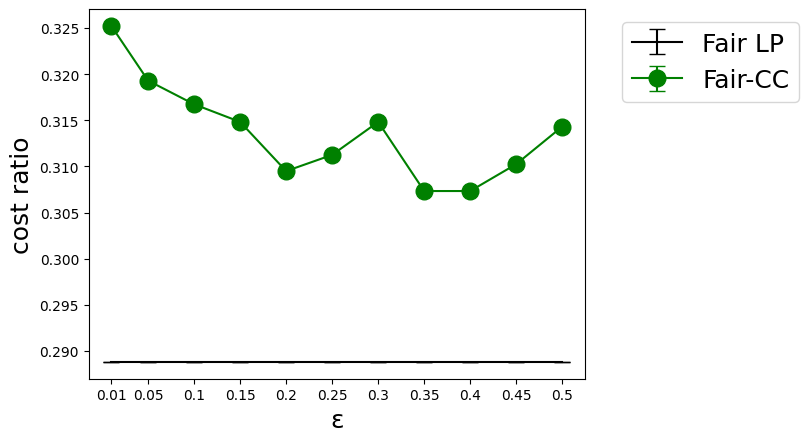}
    \includegraphics[width=0.45\columnwidth]{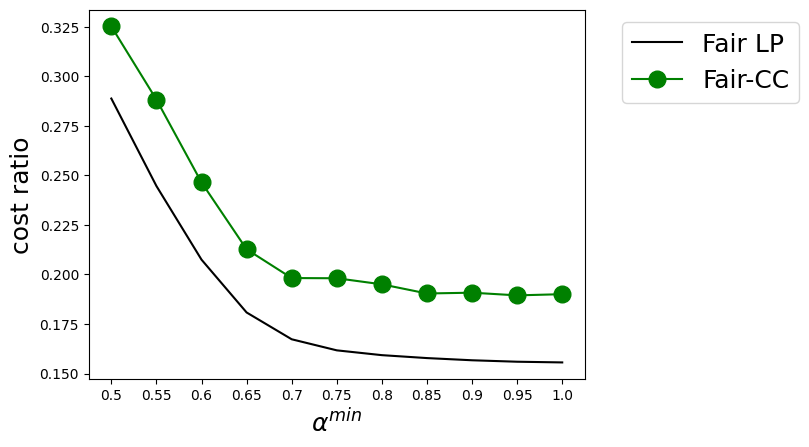}
  \caption{\victorian $\theta = 0.5$, cost ratios of our algorithm (fair-CC) and fair LP for varying $\epsilon$. }
\end{figure}

\begin{figure}[!ht]
  \centering
    \includegraphics[width=0.45\columnwidth]{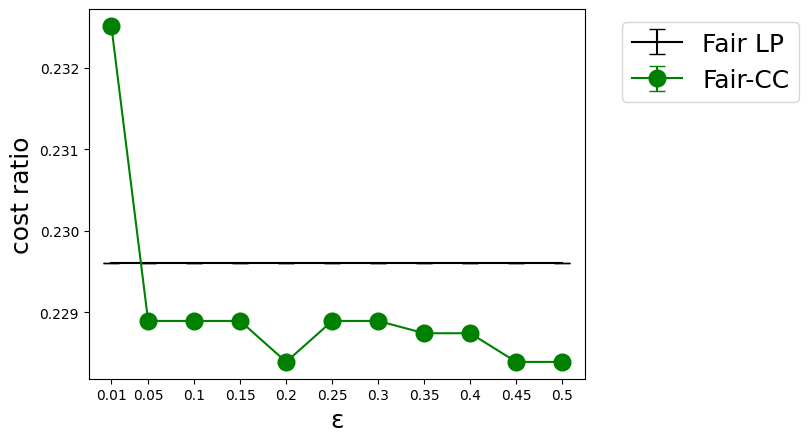}
    \includegraphics[width=0.45\columnwidth]{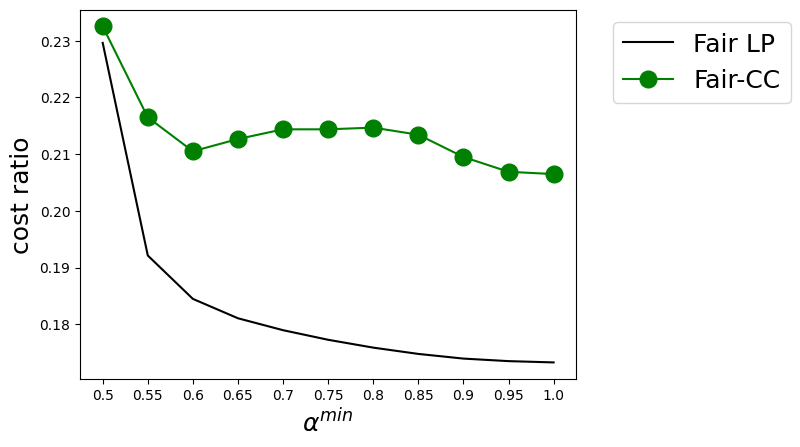}
  \caption{\victorian $\theta = 0.75$, cost ratios of our algorithm (fair-CC) and fair LP for varying $\epsilon$. }
\end{figure}

\clearpage
\subsection{Fairness Analysis}\label{corr:subsec:appendfair}
In this section we compare the
fairness violation of our algorithm (Fair-CC) with the allowed fairness violation $\epsilon$, varying $\epsilon$ from $0.01$ (with completely fair clusters) to $0.3$. We have included all the datasets except \amazon which is brought in \Cref{corr:sec:experiments}. \reuters and \victorian, each dataset has three plots associated with $\theta = 0.25, 0.5,$ and $0.75$ on the top-left, top-right, and bottom respectively. \bank, \census, and \diabetes have trivial max violations for this range of $\epsilon$ as $\alpha^{\min}$ is too low for an increase in $\epsilon$ to be able to make a change in cluster structure. There are some colors that are extremely rare in these datasets e.g. ``Amer-ind'' for race in \census.

\begin{figure}[!ht]
  \centering
    \includegraphics[width=0.45\columnwidth]{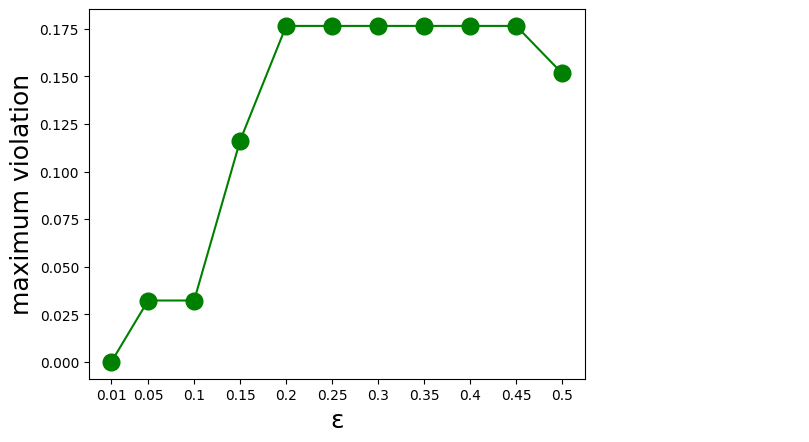}
    \includegraphics[width=0.45\columnwidth]{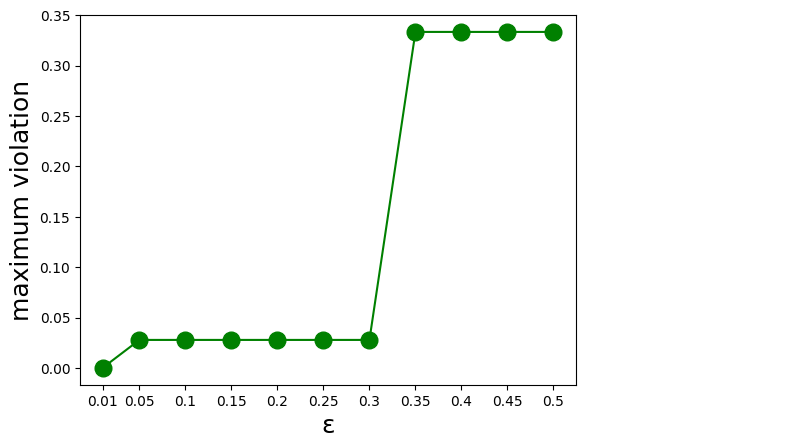}
    \includegraphics[width=0.45\columnwidth]{plots/violations/reuters_theta0.75_colors0-1_alphas1_scale0_viol_var_eps.png}
  \caption{\reuters, maximum violation of our algorithm for varying $\epsilon$. }
\end{figure}

\begin{figure}[!ht]
  \centering
    \includegraphics[width=0.45\columnwidth]{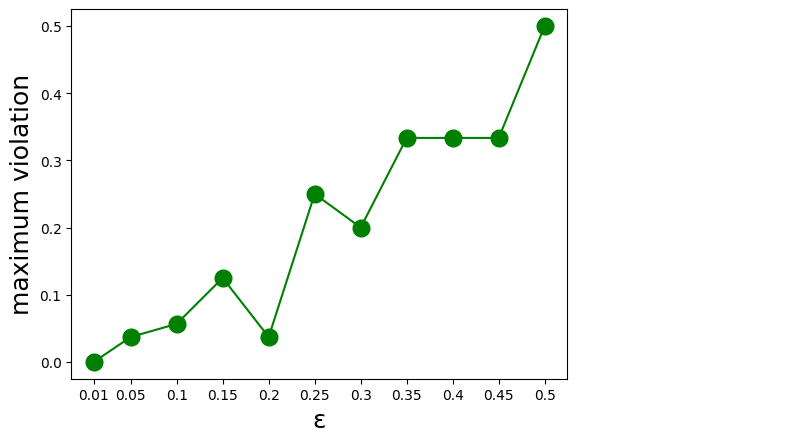}
    \includegraphics[width=0.45\columnwidth]{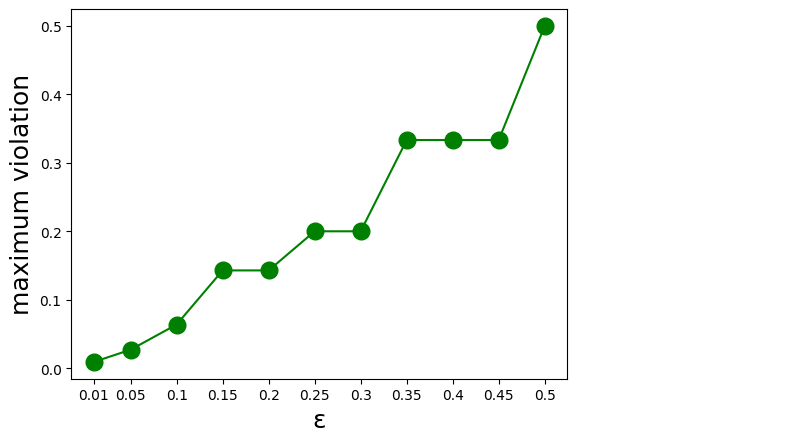}
    \includegraphics[width=0.45\columnwidth]{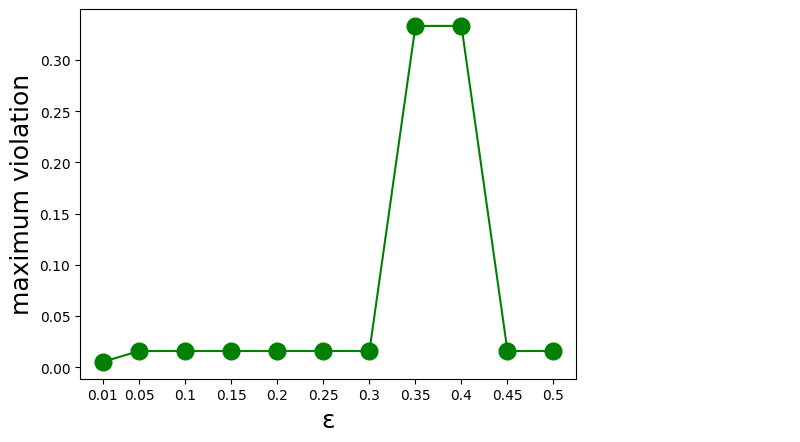}
  \caption{\victorian, maximum violation of our algorithm for varying $\epsilon$.}
\end{figure}

\end{document}